\documentclass[11pt]{article}
\usepackage{fullpage}

\input{irom.sty}

\usepackage[bookmarks=true]{hyperref}
\usepackage{amsthm}
\usepackage{algorithmic}
\usepackage{algorithm}
\usepackage{booktabs}
\usepackage{authblk}
\usepackage{breakcites}
\usepackage[labelfont={color=black}]{caption}
\usepackage[numbers]{natbib}

\newcommand{\N}{\mathcal{N}}
\newcommand{\D}{\mathcal{D}}
\newcommand{\E}{\mathcal{E}}
\newcommand{\Y}{\mathcal{Y}}
\newcommand{\HH}{\mathcal{H}}
\newcommand{\Z}{\mathcal{Z}}
\newcommand{\PP}{\mathcal{P}}
\newcommand{\revision}[1]{#1} 
\newcommand{\revisionn}[1]{#1}

\theoremstyle{nospace} \newtheorem{theorem}{Theorem}
\theoremstyle{nospace} 
\theoremstyle{nospace} 
\theoremstyle{nospace} 
\theoremstyle{nospace} 
\theoremstyle{nospace} \newtheorem{definition}{Definition}
\theoremstyle{nospace} \newtheorem{corollary}{Corollary}
\theoremstyle{nospace} \newtheorem{assumption}{Assumption}

\title{PAC-Bayes Control: Learning Policies that \\ Provably Generalize to Novel Environments}

\author[1]{Anirudha Majumdar}
\author[2]{Alec Farid}
\author[3]{Anoopkumar Sonar}
\affil[1,2]{Department of Mechanical and Aerospace Engineering}
\affil[3]{Department of Computer Science}
\affil[ ]{Princeton University}
\affil[ ]{Princeton, NJ, 08544, USA}
\affil[ ]{Emails: \{ani.majumdar, afarid, asonar\}@princeton.edu}

\date{\today}

\begin{document}

\maketitle

\begin{abstract}
Our goal is to learn control policies for robots that provably generalize well to novel environments given a dataset of example environments. The key technical idea behind our approach is to leverage tools from \emph{generalization theory} in machine learning by exploiting a precise analogy (which we present in the form of a reduction) between generalization of control policies to novel environments and generalization of hypotheses in the supervised learning setting. In particular, we utilize the \emph{Probably Approximately Correct (PAC)-Bayes} framework, which allows us to obtain upper bounds that hold with high probability on the expected cost of (stochastic) control policies across novel environments. We propose policy learning algorithms that explicitly seek to minimize this upper bound. The corresponding optimization problem can be solved using \emph{convex} optimization (\emph{Relative Entropy Programming} in particular) in the setting where we are optimizing over a finite policy space. In the more general setting of continuously parameterized policies (e.g., neural network policies), we minimize this upper bound using stochastic gradient descent. We present simulated results of our approach applied to learning (1) reactive obstacle avoidance policies and (2) neural network-based grasping policies. \revision{We also present hardware results for the Parrot Swing drone navigating through different obstacle environments.} Our examples demonstrate the potential of our approach to provide strong generalization guarantees for robotic systems with continuous state and action spaces, complicated (e.g., nonlinear) dynamics, rich sensory inputs (e.g., depth images), and neural network-based policies. 
\end{abstract}

\section{Introduction}
\label{sec:intro}


\begin{figure}[h!]
 \centering
   \subfigure[\label{fig:drone}] 
   {\includegraphics[width=0.6\textwidth]{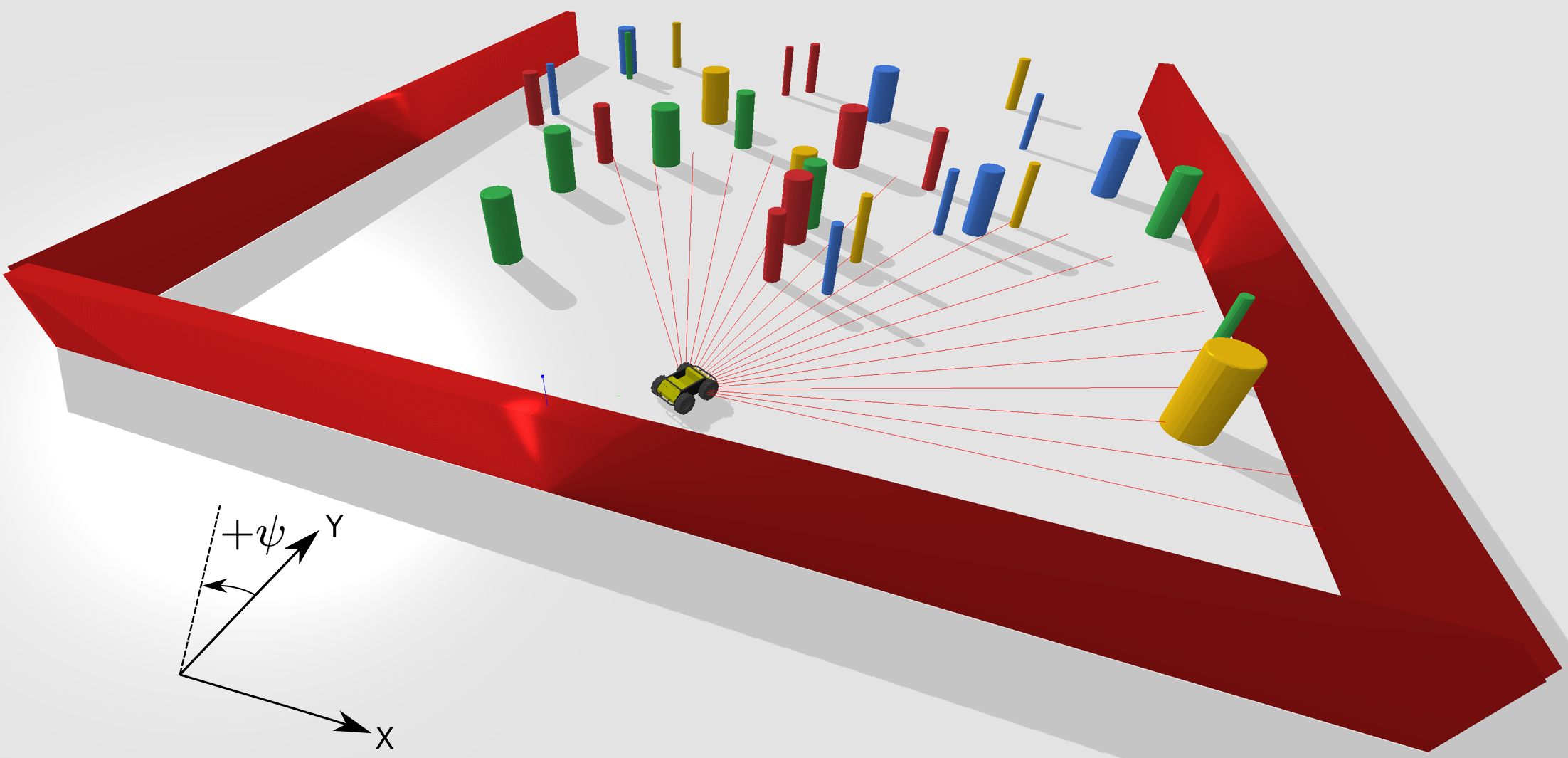} \label{fig:husky}}
   \subfigure[\label{fig:tool}] 
   {\includegraphics[width=0.258\textwidth]{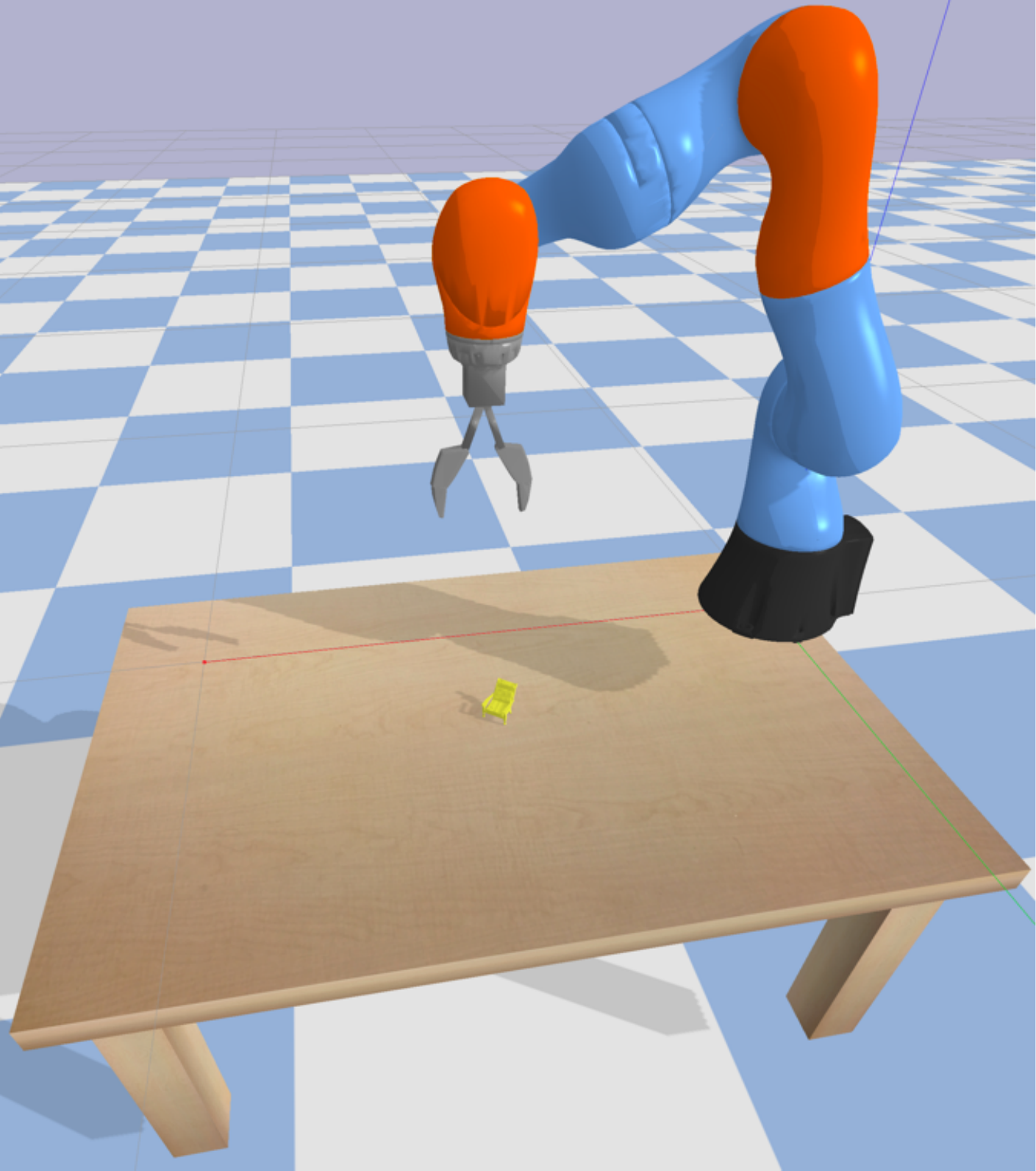} \label{fig:arm}}
	\vspace{-5pt}
\caption{\footnotesize{We demonstrate our approach for learning (i) reactive obstacle avoidance policies for a differential drive ground vehicle model equipped with a depth sensor, and (ii) neural network-based grasping policies for a manipulator model equipped with an RGB-D sensor. 
Our approach provides strong guarantees on the performance of the learned policies on novel environments even with a relatively small number of training environments (e.g., a guaranteed expected collision-free traversal rate of $87.9\%$ using 1000 training environments for the obstacle avoidance example and a guaranteed expected success rate of $70.6\%$ for the grasping example using 2000 training objects).} \label{fig:examples}}
\vspace*{-5pt}
\end{figure}

\begin{figure}[h!]
 \centering
   \subfigure[\label{fig:lab_swing}]{
    \includegraphics[width=0.4\columnwidth]{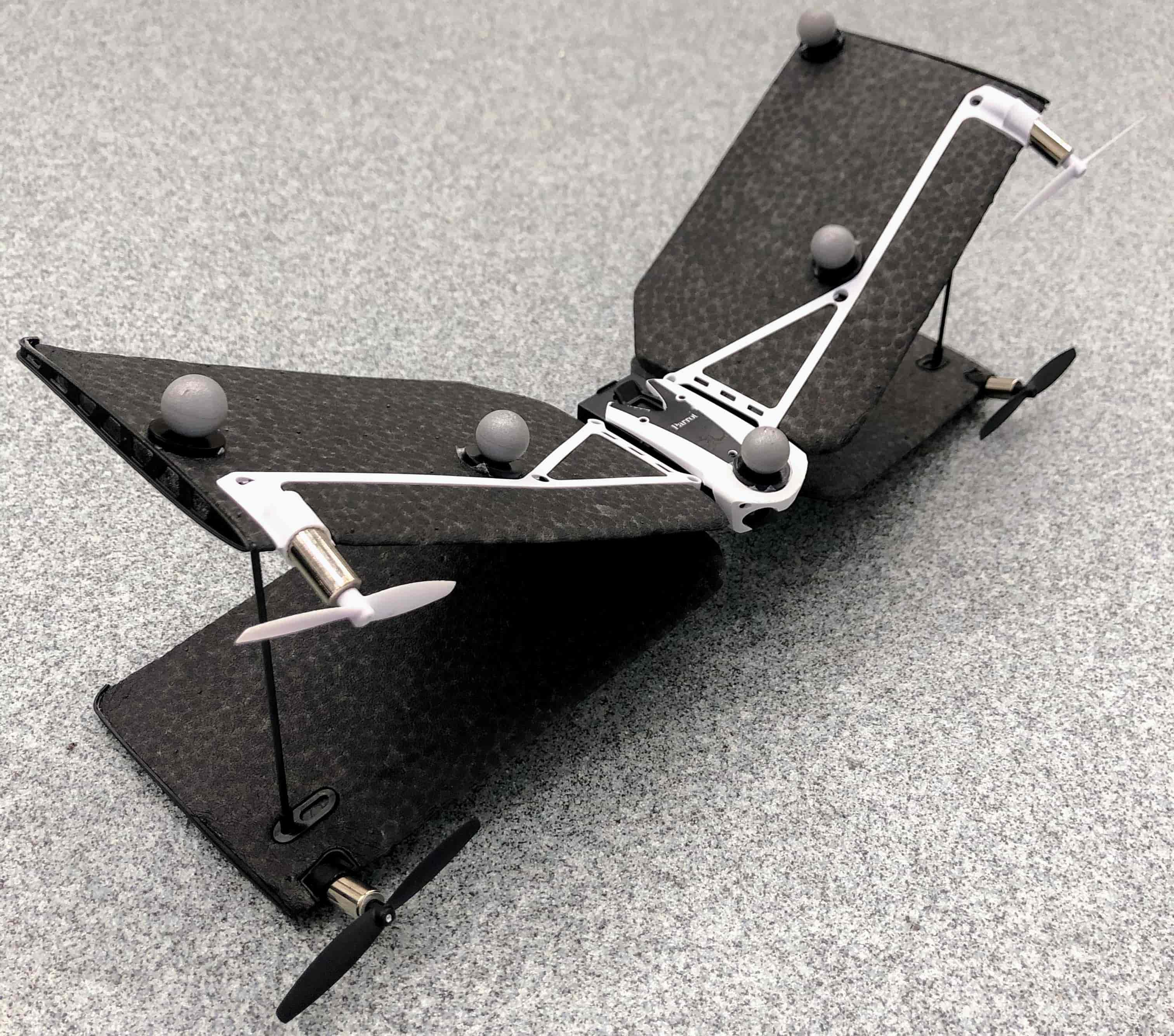}} 
   \subfigure[\label{fig:swingb}]{
    \includegraphics[width=0.428\columnwidth]{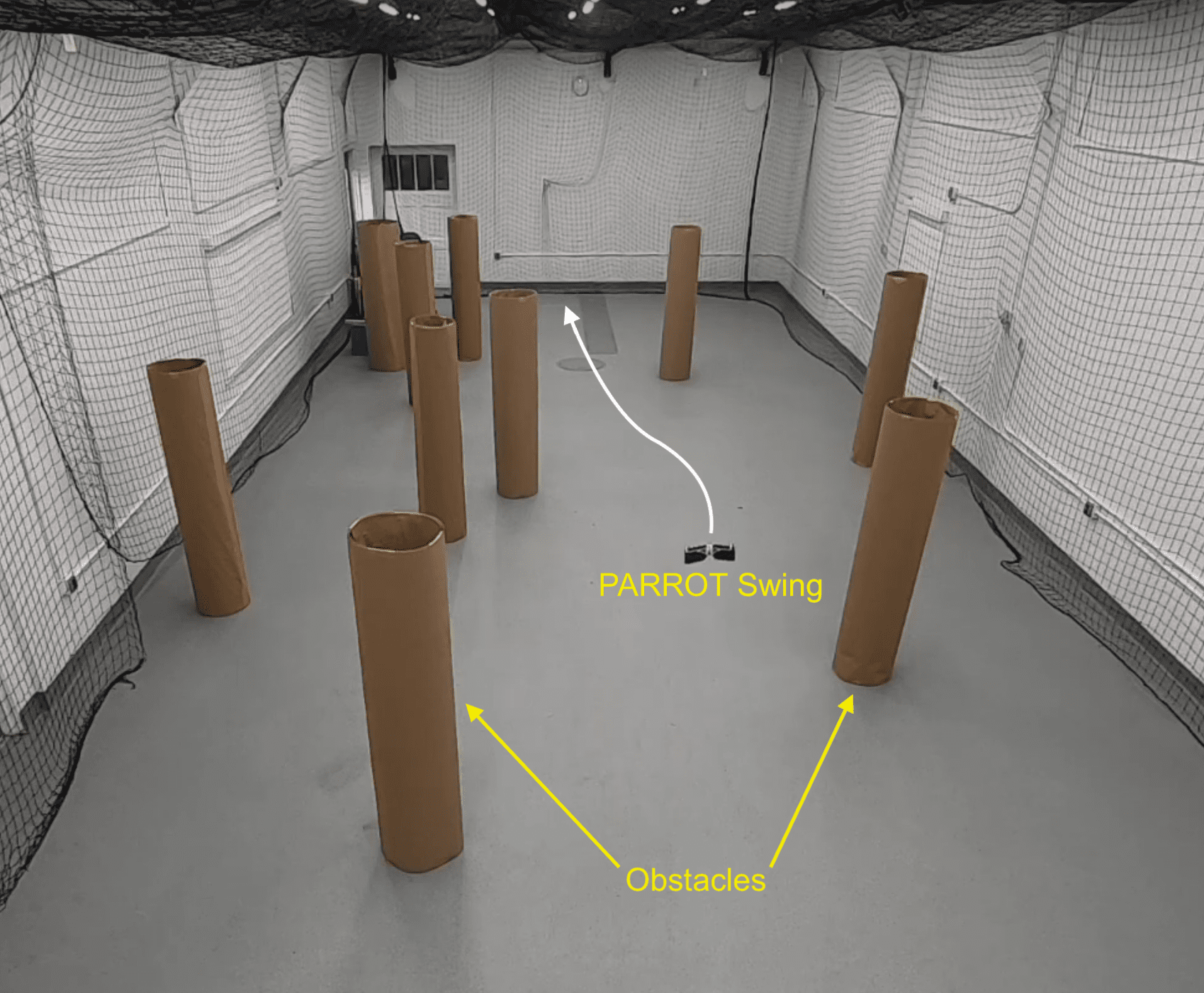}} 
   \vspace{-5pt}
\caption{\footnotesize{Pictured in (a) is a Parrot Swing drone --- a quadrotor/fixed-wing hybrid vehicle. We demonstrate our approach for learning reactive obstacle avoidance policies for the Swing given a simulated depth sensor. Our approach provides a guaranteed expected collision-free traversal rate of 88.6$\%$ on novel environments using 1000 simulated training environments. When testing on unseen environments within the netted area pictured in (b), the Swing succeeds in 18/20 trials. Videos of representative trials can be found at https://youtu.be/p5CjcSsojg8.}\label{fig:swing}}
\vspace{-10pt}
\end{figure}

Imagine an unmanned aerial vehicle that successfully navigates a thousand different obstacle environments or a robotic manipulator that successfully grasps a million objects in our dataset. How likely are these systems to succeed on a novel (i.e., previously unseen) environment or object? How can we explicitly learn control policies that provably generalize well to environments or objects that our robot has not previously encountered? Current approaches for designing control policies for robotic systems either do not provide such guarantees on generalization or provide guarantees only under very restrictive assumptions (e.g., strong assumptions on the geometry of a novel environment \citep{Schouwenaars04, Fraichard07, Althoff15, Majumdar17}). 

The goal of this paper is to develop an approach for learning control policies for robotic systems that provably generalize well with high probability to novel environments given a dataset of example environments. 
 The key conceptual idea for enabling this is to establish a precise analogy between generalization of policies to novel environments and generalization in supervised learning. This analogy allows us to translate techniques for learning hypotheses with generalization guarantees in the supervised learning setting into techniques for learning control policies for robot tasks with performance guarantees on novel environments.

In order to obtain more insight into this analogy, suppose we have a dataset of $N$ objects. A simple approach to learning a grasping policy is to synthesize one that achieves the best possible performance on these $N$ objects. However, such a strategy might result in an overly complex policy that overfits to the specific objects at hand. This is a particularly important challenge for robotics applications since datasets are generally relatively small (e.g., as compared to training sets for image classification tasks).  In order to learn a policy that generalizes well to novel environments, we may need to add a ``regularizer" that penalizes the ``complexity" of the policy. This raises the following questions: (1) what form should this regularizer take?; and (2) can we provide a formal guarantee on the performance of the resulting policy on novel environments?

The analogous questions for \emph{supervised} learning algorithms have been extensively studied in the literature on \emph{generalization theory} in machine learning. Here we leverage PAC-Bayes theory (Probably Approximately Correct Bayes) \citep{McAllester99}, which provides some of the tightest known generalization bounds for classical supervised learning approaches \citep{Langford03, Seeger02, Germain09}. Very recently, PAC-Bayes analysis has also been used to train deep neural networks with guarantees on generalization performance \citep{Dziugaite17, Neyshabur17, Neyshabur17a}. As we will see, we can leverage PAC-Bayes theory to provide precise answers to both questions posed above; it will allow us to specify a regularizer for designing (stochastic) control policies that provably generalize well (with high probability) to novel environments.

\subsection{Statement of Contributions}
\label{sec:contributions}

\revisionn{The primary contribution of this paper is to introduce a framework for providing \emph{generalization guarantees} for learning-based control of robots. While generalization bounds have been studied extensively in the literature on supervised learning (as discussed above), there has been relatively little work on this topic in the literature on robot learning (see Section \ref{sec:related work} for a thorough literature review). 
To our knowledge, the results in this paper constitute the first attempt to provide generalization guarantees on learning-based control policies for robotic systems with continuous state and action spaces, complicated (e.g., nonlinear or hybrid) dynamics, and rich sensory inputs (e.g., RGB-D images).} 
To this end, this paper makes four specific contributions. First, we provide a \emph{reduction} that allows us to translate generalization bounds for supervised learning problems to generalization bounds for control policies. We apply this reduction to translate PAC-Bayes bounds to the control setting we consider here (Section \ref{sec:pac bayes control}). Second, we propose learning algorithms that minimize the regularized cost functions specified by PAC-Bayes theory in order to synthesize control policies with generalization guarantees (Section \ref{sec:computing pac-bayes controllers}). In the setting where we are optimizing over a finite policy space (Section \ref{sec:finite policy space}), the corresponding optimization problem can be solved using \emph{convex} optimization techniques (\emph{Relative Entropy Programs (REPs)} in particular). In the more general setting of continuously-parameterized policies (Section \ref{sec:continuous policy space}), we rely on stochastic gradient descent to perform the optimization. Third, in Section \ref{sec:domain shift} we present an extension of our basic approach that allows us to learn policies that are distributionally robust (i.e., handle settings where test environments are drawn from a different distribution than training environments). Fourth, we demonstrate our approach in simulation for learning (i) depth sensor-based reactive obstacle avoidance policies for the ground robot model shown in Figure \ref{fig:husky} (Section \ref{sec:reactive obstacle avoidance control}), and (ii) neural network-based grasping policies for the manipulator model shown in Figure \ref{fig:arm} (Section \ref{sec:grasping}). \revision{Finally, we also present hardware results for reactive obstacle avoidance control with the Parrot Swing drone shown in Figure \ref{fig:lab_swing} (Section \ref{sec:hardware}).}
Our simulation and hardware results demonstrate that we are able to obtain strong generalization guarantees even with a relatively small number of training environments. We compare the bounds obtained from PAC-Bayes theory with exhaustive sampling to illustrate the tightness of the bounds. 

A preliminary version of this work \citep{Majumdar18} was presented at the Conference on Robot Learning (CoRL) 2018. In this significantly revised and extended version, we additionally present: (i) an extension of our basic approach for providing generalization guarantees in settings where test environments are drawn from a different distribution to training environments (Section \ref{sec:domain shift}), (ii) an application of our framework for learning neural-network based grasping policies (Section \ref{sec:grasping}), (iii) a method for handling stochastic dynamics (Section \ref{sec:stochastic rollouts}), (iv) hardware implementation of the depth sensor-based reactive obstacle avoidance policies (Section \ref{sec:hardware}), and (v) a more thorough discussion of challenges associated with our approach and promising future directions (Section \ref{sec:conclusions}).

\subsection{Related Work}
\label{sec:related work}

One approach for synthesizing control policies with guaranteed performance is to leverage robust control techniques (e.g., H-infinity control \citep{Francis87} or chance-constrained programming \citep{Charnes59, Blackmore06, Vitus11, Ono15}). However, such techniques typically require an explicit description of the uncertainty affecting the system. While uncertainty models for the robot's dynamics or measurements can often be obtained via system identification, assuming an uncertainty model for the environment (e.g., a distribution over all possible environment geometries) is unrealistic. One way to address this is to assume that a novel environment satisfies conditions that allow a real-time planner to \emph{always} succeed. For example, in the context of navigation, this constraint could be satisfied by hand-coding emergency maneuvers (e.g., stopping maneuvers or loiter circles) that are always guaranteed to succeed \citep{Schouwenaars04, Fraichard07, Althoff15}. However, requiring the existence of such emergency maneuvers can lead to extremely conservative behavior. Another approach is to assume that the environment satisfies certain geometric conditions (e.g., large separation between obstacles) that allow for safe navigation \citep{Majumdar17}. However, such conditions are rarely satisfied by real-world environments. Moreover, such conditions are domain specific; it is not clear how one would specify such constraints for problems other than navigation (e.g., grasping). 

Another conceptually appealing approach for synthesizing policies with guaranteed performance on a priori unknown environments is to model the problem as a Partially Observable Markov Decision Process (POMDP) \citep{Kaelbling98}, where the environment is part of the (partially observed) state of the system \citep{Richter15}. Computational considerations aside, such an approach is made infeasible by the need to specify a distribution over environments the robot might encounter. Unfortunately, specifying such a distribution over real-world environments is an extremely challenging endeavor. Thus, many approaches (including ours) assume that we only have \emph{indirect} access to the true underlying distribution over environments in the form of examples. For example, \citep{Richter15, Richter17} propose an approximation to the POMDP framework in the context of navigation by learning to predict future collision probabilities from past data. The work on deep-learning based approaches for control represents another prominent set of techniques where interactions with example environments are used to learn control policies (see, e.g., \citep{Lenz15, Levine16, Agrawal16, Mahler17, Tobin17, Gupta17, Gupta17a, Zhu17, Sunderhauf18}). While the approaches mentioned above have led to impressive empirical demonstrations, it is challenging to guarantee that such methods will perform well on environments that are not part of the training data (especially when a limited number of training examples are available, as is often the case for robotics applications). Our work seeks to address this challenge using ideas from generalization theory. 

The primary theoretical framework we utilize in this paper is PAC-Bayes generalization theory \citep{McAllester99}. PAC-Bayes theory provides some of the tightest known generalization bounds for classical supervised learning problems \citep{Langford03, Seeger02, Germain09} and has recently been applied to explain and promote generalization in deep learning \citep{Dziugaite17, Neyshabur17, Neyshabur17a}. PAC-Bayes theory has also been applied to learn control policies for Markov Decision Processes (MDPs) with provable sample complexity bounds  \citep{Fard10, Fard12}. These approaches also exploit the intuition (see Section \ref{sec:intro}) that ``regularizing" policies in an appropriate manner can prevent overfitting and lead to sample efficiency (see also \citep{Neu17, Kearns00, Bagnell01, Bagnell04, Schulman15} for other approaches that exploit this intuition in the reinforcement learning context). However, we note that the focus of our work is quite different from the work on PAC-Bayes MDP bounds (and the more general framework of PAC MDP bounds \citep{Kearns02, Brafman02, Fu14}), which consider the standard reinforcement learning setup where a control policy must be learned through multiple interactions with a given MDP (with unknown transition dynamics and/or rewards). 
In contrast, here we focus on \emph{zero-shot} generalization to a novel environment (e.g., obstacle environments or objects). In other words, a policy learned from examples of different environments must immediately perform well on a new one (i.e., without further exploratory interactions with the new environment). We further note that \citep{Fard10} considers finite state and action spaces along with policies that depend on full state feedback while \citep{Fard12} relaxes the assumption on finite state spaces but retains the other modeling assumptions. In contrast, we target systems with continuous state and action spaces and synthesize control policies that rely on rich sensory inputs. 

On the algorithmic front, we make significant use of Relative Entropy Programs (REPs) \citep{Chandrasekaran17}. REPs constitute a rich class of \emph{convex} optimization problems that generalize many other problems including linear programs, geometric programs, and second-order cone programs \citep{Boyd04}. REPs are optimization problems in which a linear functional of the decision variables is minimized subject to linear constraints and conic constraints given by a \emph{relative entropy cone}. REPs are amenable to efficient solution techniques (e.g., interior point methods \citep{Nesterov94}) and can be solved using existing software packages (e.g., Mosek \citep{mosek}, SCS \citep{SCS, ODonoghue16}, and ECOS  \citep{ECOS}). We refer the reader to \citep{Chandrasekaran17} for a more thorough introduction to REPs. Importantly for us, REPs can handle constraints of the form $\KL(p \| q) \leq c$, where $p$ and $q$ are decision variables corresponding to probability vectors, $\KL(\cdot \|| \cdot)$ represents the Kullback-Leibler divergence, and $c$ is a scalar decision variable. As we will see, this allows us to use REPs to learn control policies using the PAC-Bayes framework in the setting where we are optimizing over a finite set of policies. 

\subsection{Notation}

We use the notation $v[i]$ to refer to the i-th component of a vector $v \in \RR^n$. We use $\RR^n_+$ to denote the set of elementwise nonnegative vectors in $\RR^n$, $\mathbb{Z}_+$ to denote nonnegative integers, and $\odot$ to denote element-wise multiplication.

\section{Problem Formulation}
\label{sec:problem formulation}

We assume that the robot's dynamics are described by a discrete-time system:
\begin{equation}
\label{eq:dynamics}
x(t+1) = f(x(t), u(t); E),
\end{equation}
where $t \in \mathbb{Z}_+$ is the time index, $x(t) \in \X$ is the state at time $t$, $u(t) \in \U$ is the control input at time $t$, and $E$ is the environment that the robot operates in. We use the term ``environment" here broadly to refer to any factors that are external to the robot. For example, $E$ could refer to an obstacle field that a mobile robot is attempting to navigate through, external disturbances (e.g., wind gusts) that a UAV is subjected to, or an object that a manipulator is attempting to grasp. 

Let $\E$ denote the space of all possible environments. We then make the following assumption.

\begin{assumption}
\label{assumption environments}
There is an underlying distribution $\D$ over $\E$ from which environments are drawn.
\end{assumption}

Importantly, we \emph{do not} assume that we have explicit descriptions of $\E$ or $\D$. Instead, we only assume indirect access to $\D$ in the form of a dataset $S = \{E_1, \dots, E_N\}$ of $N$ training environments drawn i.i.d. from $\D$. In Section \ref{sec:domain shift}, we will present an extension of our basic framework that allows us to relax this assumption and handle settings where training and test environments are drawn from different distributions. 

Let $g: \X \times \E \rightarrow \Y$ denote the robot's sensor mapping from a state $x$ and an environment $E$ to an observation $y = g(x; E) \in \mathcal{Y}$. \revision{Since we are interested in partially observable settings, we do not make any particular assumptions (e.g., injectivity or bijectivity) on the sensor mapping $g$.} Let $\pi: \Y \rightarrow \U$ denote a control policy that maps sensor measurements to control inputs. Note that this is a very general model and can capture control policies that depend on \emph{histories} of sensor measurements (by simply augmenting the state to keep track of histories of states and letting $\Y$ denote the space of histories of sensor measurements). 

We assume that the robot's desired behavior is encoded through a cost function. In particular, let $r_\pi: \E \rightarrow (\X \times \U)^T$ denote the function that ``rolls out" the system with control policy $\pi$, i.e., $r_\pi$ maps an environment $E$ to the state-control trajectory one obtains by applying the control policy $\pi$ (up to a time horizon $T$). We will assume that the environment captures all sources of stochasticity (including random initial conditions) and the rollout function for a \emph{particular} environment is thus deterministic (we discuss the case of stochastic rollouts in Section \ref{sec:stochastic rollouts}). 
We then let $C(r_\pi; E)$ denote the cost incurred by control policy $\pi$ when operating in environment $E$ over a time horizon $T$. We assume that the cost $C(r_\pi; E)$ is bounded and will assume (without further loss of generality) that $C(r_\pi; E) \in [0,1]$.
We make the following important assumption in this work.

\begin{assumption}
\label{assumption}
Given any control policy $\pi$, we can compute the cost $C(r_\pi; E_i)$ for the training environments $E_1, \dots, E_N$.
\end{assumption}

This assumption is satisfied if one can simulate the robot's operation in the environments $E_1, \dots, E_N$. We note that computational considerations aside, we do not make any restrictions on the dynamics $f$ or the sensor mapping $g$ beyond the ability to simulate them. The models that our approach can handle are thus extremely rich in principle (e.g., nonlinear or hybrid dynamics, sensor models involving raycasting or simulated vision, etc.).  

Another possibility for satisfying Assumption \ref{assumption} is to run the policy $\pi$ on the hardware system itself in the given environments. This may be a feasible option for problems such as grasping, which are not safety-critical in nature. In such cases, our approach does not require models of the dynamics, sensor mapping, or the rollout function.

{\bf Goal:} Our goal is to design a control policy that minimizes the expected value of the cost $C$ across environments:
\begin{equation}
\label{eq:opt_det}
 \underset{\pi \in \Pi}{\textrm{min}} \  \ C_\D(\pi) :=  \underset{\pi \in \Pi}{\textrm{min}} \ \  \underset{E \sim \D}{\EE} \ [C(r_\pi; E)].
\end{equation}
In this work, it will be useful to consider a more general setting where we choose a \emph{distribution} $P$ over the control policy space $\Pi$ instead of making a single deterministic choice. This is because the PAC-Bayes bounds we use will assume this setting.
Our goal is then to solve the following optimization problem, which we refer to as $\OPT$:
\begin{flalign*}
&& C^\star := \underset{P \in \PP}{\textrm{min}} \  \ C_\D(P) := \underset{P \in \PP}{\textrm{min}} \  \  \underset{E \sim \D}{\EE} \ \underset{\pi \sim P}{\EE} [C(r_\pi; E)], &&& (\OPT)
\end{flalign*}
where $\PP$ denotes the space of probability distributions over $\Pi$. Note that the outer expectation here is taken with respect to the \emph{unknown} distribution $\D$. This constitutes the primary challenge in tackling this problem. 


\section{Background}
\label{sec:background}

The primary technical framework we leverage in this paper is PAC-Bayes theory. In Section \ref{sec:pac bayes learning}, we provide a brief overview of the key results from PAC-Bayes theory in the context of supervised learning. We first provide some brief background on the properties of the Kullback-Leibler (KL) divergence in Section \ref{sec:kl divergence} and show how we can compute its inverse using Relative Entropy Programming (REP) in Section \ref{sec:kl inverse}.

\subsection{KL divergence}
\label{sec:kl divergence}

Given two discrete probability distributions $P$ and $Q$ defined over a common set, the KL divergence from Q to P is defined as
\begin{equation}
\KL(P \| Q) := \sum_i P[i] \log \Bigg( \frac{P[i]}{Q[i]} \Bigg).
\end{equation}
For scalars $p,q \in [0,1]$, we define
\begin{equation}
\label{eq:kl bernoulli}
\KL(p \| q) := \KL(B(p) \| B(q)) = p \log \frac{p}{q} + (1 - p) \log \frac{1 - p}{1 - q},
\end{equation}
where $B(p)$ denotes a Bernoulli distribution on $\{0,1\}$ with parameter (i.e., mean) $p$. 

For distributions P and Q of a continuous random variable, the KL divergence is defined to be
\begin{equation}
\KL(P\|Q)=\int p(x)\,\log {\frac {p(x)}{q(x)}}\,dx,
\end{equation}
where $p$ and $q$ denote the densities of $P$ and $Q$. Importantly, if $P$ and $Q$ correspond to normal distributions $N_p = \mathcal{N}(\mu_p, \Sigma_p)$ and $N_q = \mathcal{N}(\mu_q, \Sigma_q)$ over $\RR^d$, the KL divergence can be computed in closed form as
\begin{equation}
\label{eq:KL normal}
\KL(N_p \| N_q) = \frac{1}{2} \Bigg( \textrm{Tr}(\Sigma_q^{-1} \Sigma_p) + (\mu_q - \mu_p)^T \Sigma_q^{-1} (\mu_q - \mu_p) + \log \frac{\textrm{det}(\Sigma_q)}{\textrm{det}(\Sigma_p)} - d \Bigg).
\end{equation}

\subsubsection{Computing KL inverse using Relative Entropy Programming}
\label{sec:kl inverse}

PAC-Bayes bounds (Section \ref{sec:pac bayes learning}) are typically expressed as bounds on a quantity $q^\star \in [0,1]$ of the form $\KL(p \| q^\star) \leq c$ (for some $p \in [0,1]$ and $c \geq 0$). These bounds can then be used to upper bound $q^\star$ by the \emph{KL inverse} as follows:
\begin{equation}
\label{eq:kl inverse}
q^\star \leq \KL^{-1} (p \| c) := \sup \{ q \in [0,1] \ | \ \KL(p \| q) \leq c \}.
\end{equation}

In prior work on PAC-Bayes theory, the KL inverse was numerically approximated using local root-finding techniques such as Newton's method \citep{Dziugaite17, Dziugaite17a}, which do not have a priori guarantees on convergence to a global solution. Here we observe that the KL inverse is readily expressed as the optimal value of a simple Relative Entropy Program (ref. Section \ref{sec:related work}). In particular, the expression for the KL inverse in \eqref{eq:kl inverse} corresponds to an optimization problem with a (scalar) decision variable $q$, a linear cost function (i.e., $-q$), linear inequality constraints (i.e., $0 \leq q \leq 1$), and a constraint on the KL divergence between the decision variable $q$ and the constant $p$.
We can thus compute the KL inverse exactly (up to numerical tolerances) using convex optimization (e.g., interior point methods \citep{Chandrasekaran17}).

\subsection{PAC-Bayes Theory in Supervised Learning}
\label{sec:pac bayes learning}

We now provide a brief overview of the key results from PAC-Bayes theory in the context of supervised learning.
Let $\Z$ be an input space and $\Z'$ be a set of labels. Let $\D$ be the (unknown) true distribution on $\Z$. Let $\HH$ be a hypothesis class consisting of functions $h_w: \Z \rightarrow \Z'$ parameterized by $w \in \RR^d$ (e.g., neural networks parameterized by weights $w$). Let $l: \HH \times \Z \rightarrow \RR$ be a loss function\footnote{Note that we are considering a slightly restricted form of the supervised learning problem where each input $z \in \Z$ has only one correct label $z' \in \Z'$. The loss thus only depends on the input $z$ and the label $h_w(z)$. The PAC-Bayes framework applies to the more general setting where there is an underlying true distribution on $\Z \times \Z'$ and the loss thus has the form $l: \HH \times \Z \times \Z' \rightarrow \RR$. However, the more restricted setting is sufficient for our needs here.}. We will denote by $\PP$ the space of probability distributions on the parameter space $\RR^d$. Informally, we will refer to distributions on $\HH$ when we mean distributions over the underlying parameter space.

PAC-Bayes analysis then applies to learning algorithms that output a \emph{distribution} over hypotheses. Specifically, the PAC-Bayes framework applies to learning algorithms with the following structure:
\begin{enumerate}
\item Choose a \emph{``prior" distribution} $P_0 \in \PP$ \emph{before} observing any data.
\item Observe training data samples $S = \{z_i\}_{i=1}^N$ and choose a \emph{posterior distribution} $P \in \PP$. This posterior can depend on the data and the prior. 
\end{enumerate}
It is important to note that the posterior distribution $P$ \emph{need not} be the Bayesian posterior. PAC-Bayes theory applies to \emph{any} distribution $P$.

Let us denote the training loss associated with the posterior distribution $P$ as:
\begin{equation}
l_S(P) := \frac{1}{N} \sum_{z \in S} \underset{w \sim P}{\EE} [l(h_w; z)],
\end{equation}
and the true expected loss as:
\begin{equation}
l_\D(P) := \underset{z \sim \D}{\EE} \ \underset{w \sim P}{\EE} [l(h_w; z)].
\end{equation}
The following theorem is the primary result from PAC-Bayes theory\footnote{The bound we state here is due to Maurer \citep{Maurer04} and improves slightly upon the original PAC-Bayes bounds \citep{McAllester99}. The stated bound holds when costs are bounded in the range $[0,1]$ (as assumed here) and we have $N \geq 8$ samples.
}. 

\begin{theorem}[PAC-Bayes Bound for Supervised Learning \citep{McAllester99, Maurer04}]
\label{thm:pac bayes kl}
For any $\delta \in (0,1)$, with probability at least $1 - \delta$ over samples $S \sim \D^N$, the following inequality holds:
\begin{equation}
\label{eq:pac bayes kl}
\KL(l_S(P) \| l_\D(P)) \leq \frac{\KL(P \| P_0) + \log(\frac{2 \sqrt{N}}{\delta})}{N}.
\end{equation}
\end{theorem}
Here, $\KL(l_S(P) \| l_\D(P))$ is interpreted as a KL divergence between Bernoulli distributions and computed using \eqref{eq:kl bernoulli} (this is meaningful since $l_S(P)$ and $l_\D(P)$ are scalars bounded within $[0,1]$). 

Intuitively, Theorem \ref{thm:pac bayes kl} provides a bound on how ``close'' the training loss $l_S(P)$ and the true expected loss $l_\D(P)$ are. However, in practice, one would like to find an \emph{upper bound} on the true expected loss $l_\D(P)$. Such an upper bound can be obtained by computing the KL inverse (ref. Section \ref{sec:kl inverse}):
\begin{equation}
\label{eq:pac bayes learning kl version}
l_\D(P) \leq \KL^{-1}\Big(l_S(P) \| \frac{\KL(P \| P_0) + \log(\frac{2 \sqrt{N}}{\delta})}{N} \Big).
\end{equation}

Another upper bound that is useful for the purpose of optimization is provided by the following corollary, which follows from Theorem \ref{thm:pac bayes kl} by applying the well-known upper bound for the KL inverse one obtains by applying Pinsker's inquality: $\KL^{-1}(p \| c) \leq p + \sqrt{c/2}$.

\begin{corollary}[PAC-Bayes Upper Bound for Supervised Learning \citep{McAllester99, Maurer04}]
\label{thm:pac bayes learning}
For any $\delta \in (0,1)$, with probability at least $1 - \delta$ over samples $S \sim \D^N$, the following inequality holds:
\begin{equation}
\label{eq:pac bayes upper bound}
\underbrace{l_\D(P)}_{\text{True expected loss}} \leq \underbrace{l_S(P)}_{\text{Training loss}} + \underbrace{\sqrt{\frac{\KL(P \| P_0) + \log(\frac{2 \sqrt{N}}{\delta})}{2N}}}_{\text{``Regularizer"}}.
\end{equation}
\end{corollary}
Corollary \ref{thm:pac bayes learning} provides a strategy for choosing a distribution $P$ over hypotheses with a provable guarantee on generalization: minimize the right hand side (RHS) of inequality \eqref{eq:pac bayes upper bound} consisting of the training loss and a ``regularization" term. 

\section{PAC-Bayes Control}
\label{sec:pac bayes control}


We now describe our approach for adapting the PAC-Bayes framework in order to tackle the policy learning problem $\OPT$ and synthesize (stochastic) control policies with guaranteed expected performance across novel environments. 
Our key idea for doing this is to exploit a precise analogy between the supervised learning setting from Section \ref{sec:pac bayes learning} and the policy learning setting described in Section \ref{sec:problem formulation}. Table \ref{tab:analogy} presents this relationship.  
   
\begin{table}[ht]
\centering
\begin{tabular}{@{\extracolsep{4pt}}llcccll}
\toprule   
\multicolumn{2}{c}{Supervised Learning}  & & & & \multicolumn{2}{c}{Policy Learning}\\
 \cmidrule{1-2} 
 \cmidrule{6-7} 
Input data & $z \in \Z$ & & & & Environment & $E \in \E$ \\ 
Hypothesis & $h_w: \Z \rightarrow \Z'$ & &  {\Huge $\leftarrow$} & & Rollout function & $r_\pi: \E \rightarrow (\X \times \U)^H$ \\ 
Loss  & $l(h_w; z)$ & & & & Cost & $C(r_\pi; E)$ \\
\bottomrule
\end{tabular}
\vspace{7pt}
\caption{A reduction from the control policy learning problem we consider here to the supervised learning setting.} 
\label{tab:analogy}
\end{table}

One can think of the relationship in Table \ref{tab:analogy} as providing a \emph{reduction} from the policy learning problem $\OPT$ to a supervised learning problem. We are provided input data in the form of a data set of example environments. Choosing a ``hypothesis" corresponds to choosing a control policy $\pi$ (since the rollout function $r_\pi$ is determined by $\pi$). A ``hypothesis" maps an environment $E$ to a ``label", corresponding to the state-control trajectory obtained by applying $\pi$ on $E$. This ``label" incurs a loss $C(r_\pi; E)$. 

We can use this reduction to translate the PAC-Bayes theorems for supervised learning (Theorem \ref{thm:pac bayes kl} and Corollary \ref{thm:pac bayes learning}) to the control setting. Similar to the supervised learning setting, we assume that the space $\Pi$ of control policies is parameterized by $w \in \RR^d$. This in turn produces a parameterization of rollout functions. With a slight abuse of notation, we will refer to rollout functions $r_w$ instead of $r_\pi$ (with the understanding that $w$ is the parameter vector for the control policy $\pi$).

Let $P_0$ be a ``prior" distribution over the parameter space $\RR^d$ chosen before seeing any example environments. The prior can be used to encode domain knowledge, but need not be ``true'' in any Bayesian sense (i.e., bounds will hold for any prior). Let $P$ be a (possibly data-dependent) ``posterior". Following the notation from Section \ref{sec:problem formulation}, we denote the true expected cost across environments by $C_\D(P)$. We will denote the cost on the training environments as
\begin{equation}
C_S(P) := \frac{1}{N} \sum_{E \in S} \underset{w \sim P}{\EE} [C(r_w; E)].
\end{equation}
The following theorem is then an exact analogy of Corollary \ref{thm:pac bayes learning}.

\begin{theorem}[PAC-Bayes Bound for Control Policies]
\label{thm:pac bayes control}
For any $\delta \in (0,1)$, with probability at least $1 - \delta$ over sampled environments $S \sim \D^N$, the following inequality holds:
\begin{equation}
\label{eq:pac inequality control}
\underbrace{C_\D(P)}_{\text{True expected cost}} \leq \ C_{\textrm{PAC}}(P) := \underbrace{C_S(P)}_{\text{Training cost}} + \underbrace{\sqrt{\frac{\KL(P \| P_0) + \log(\frac{2 \sqrt{N}}{\delta})}{2N}}}_{\text{``Regularizer"}}.
\end{equation}
\end{theorem}
\begin{proof}
The proof follows immediately from Corollary \ref{thm:pac bayes learning} given the reduction in Table \ref{tab:analogy}.
\end{proof}

This theorem will constitute our primary tool for learning policies with guarantees on their expected performance across novel environments. In particular, the left hand side of inequality \eqref{eq:pac inequality control} is the cost function $C_\D(P)$ of the optimization problem $\OPT$. Theorem \ref{thm:pac bayes control} thus provides an upper bound (that holds with probability $1 - \delta$) on the true expected performance across environments of any policy distribution $P$ in terms of the loss on the sampled environments in $S = \{E_i\}_{i=1}^N$ and a ``regularizer". Our approach for choosing $P$ is to minimize this upper bound. Algorithm \ref{a:pac bayes control} outlines the steps involved in our approach. 

We note that while $P$ is chosen by optimizing $C_{\textrm{PAC}}(P)$ (i.e., the RHS of inequality \eqref{eq:pac inequality control}), the final upper bound $C_{\textrm{bound}}^\star$ on $C_\D(P)$ is not computed as $C_{\textrm{PAC}}(P^\star_\text{PAC})$. While this is a valid upper bound, a tighter bound is provided by inequality \eqref{eq:pac bayes learning kl version}. The observations made in Section \ref{sec:kl inverse} allow us to compute this final bound using a REP. This is the bound we report in the results presented in Section \ref{sec:examples}.

\begin{algorithm}[h]
	\algsetup{linenosize=\normalsize}
  \small
  \caption{{\small PAC-Bayes Policy Learning}}
  \label{a:pac bayes control}
  \begin{algorithmic}[1]
    \STATE Fix prior distribution $P_0 \in \PP$ over policies
    \STATE {\bf Inputs:} $S = \{E_1, \dots, E_N\}$: Training environments, $\delta$: Probability threshold  
    \STATE {\bf Outputs:}
    \STATE $P^\star_\text{PAC} = \underset{P \in P}{\textrm{argmin}} \ C_{\textrm{PAC}}(P) := \frac{1}{N} \sum_{E \in S} \underset{w \sim P}{\EE} [C(r_w; E)] + \sqrt{\frac{\KL(P \| P_0) + \log(\frac{2 \sqrt{N}}{\delta})}{2N}}$ 
    \STATE $C_{\textrm{bound}}^\star := \KL^{-1}\Big(C_S(P^\star_\text{PAC}) \| \frac{\KL(P^\star_\text{PAC} \| P_0) + \log(\frac{2 \sqrt{N}}{\delta})}{N} \Big)$
  \end{algorithmic}
\end{algorithm}

\section{Computing PAC-Bayes Control Policies}
\label{sec:computing pac-bayes controllers}

We now describe how to tackle the optimization problem in Algorithm \ref{a:pac bayes control} for minimizing the upper bound on the true expected cost. We will first discuss the setting where the control policy space $\Pi$ is finite (Section \ref{sec:finite policy space}). For this setting, the optimization problem can be solved to global optimality via Relative Entropy Programming. We then tackle the more general setting where $\Pi$ is continuously parameterized in Section \ref{sec:continuous policy space}.

\subsection{Finite Control Policy Space}
\label{sec:finite policy space}

Let the space of policies be $\Pi = \{\pi_1, \dots, \pi_L\}$. Our goal is then to optimize a \emph{discrete} probability distribution $P$ (with corresponding probability vector $p$) over the space $\Pi$. Thus, $p[j]$ denotes the probability assigned to policy $\pi_j$. Define a matrix $\hat{C}$ of costs, where each element
\begin{equation}
\hat{C}[i,j] = C(r_{\pi_j}; E_i)
\end{equation}
corresponds to the cost incurred on environment $E_i \in S$ by policy $\pi_j \in \Pi$ (recall that Assumption \ref{assumption} implies that we can compute each $\hat{C}[i,j]$). The training cost from inequality \eqref{eq:pac inequality control} can then be written as:
\begin{equation}
\frac{1}{N} \sum_{E \in S} \underset{\pi \sim P}{\EE} [C(r_\pi; E)] = \frac{1}{N} \sum_{i=1}^N \sum_{j=1}^L \hat{C}[i,j] p[j] := \bar{C}p,
\end{equation}
where the matrix $\bar{C}$ is defined as:
\begin{equation}
\bar{C} := \frac{1}{N} {\bf 1}^T \hat{C}.
\end{equation}
Here, ${\bf 1}$ is the all-ones vector of size $N \times 1$. We note that finding a vector $p$ that minimizes the training cost corresponds to solving a \emph{linear program}.

Minimizing the PAC-Bayes upper bound $C_{\textrm{PAC}}(P)$ corresponds to solving the following optimization problem:

\begin{flalign}
		\label{opt:finite}
		& & \underset{p \in \RR^L}{\textrm{min}} \hspace*{1cm} & \bar{C}p + \sqrt{\frac{\KL(p \| p_0) + \log(\frac{2 \sqrt{N}}{\delta})}{2N}} && \\
		& & \text{s.t.} \hspace*{1cm} & 0 \leq p \leq 1, \ \sum_j p[j] = 1. && \nonumber
\end{flalign}
This optimization problem can be \emph{equivalently} reformulated via an \emph{epigraph constraint} \citep{Boyd04} as:
\begin{flalign}
		& & \underset{p \in \RR^L, \tau}{\textrm{min}} \hspace*{1cm} & \tau && \nonumber \\
		& & \text{s.t.} \hspace*{1cm} & \tau \geq \bar{C}p + \sqrt{\frac{\KL(p \| p_0) + \log(\frac{2 \sqrt{N}}{\delta})}{2N}} && \nonumber \\
		& & &  0 \leq p \leq 1, \ \sum_j p[j] = 1. \nonumber 
\end{flalign}
We further rewrite the problem as:
\begin{flalign}
		& & \underset{p \in \RR^L, \tau, \lambda}{\textrm{min}} \hspace*{1cm} & \tau && \label{opt:pac bayes final}  \\
		& & \text{s.t.} \hspace*{1cm} & \lambda^2 \geq \frac{\KL(p \| p_0) + \log(\frac{2 \sqrt{N}}{\delta})}{2N} && \nonumber \\
		& & & \lambda = \tau - \bar{C}p, \ \lambda \geq 0 && \nonumber \\
		& & &  0 \leq p \leq 1, \ \sum_j p[j] = 1. \nonumber 
\end{flalign}

Our key observation here is that for a \emph{fixed} $\lambda = \lambda_0$, the above problem is a Relative Entropy Program (REP) since it consists of minimizing a linear cost function subject to linear equality and inequality constraints and an additional inequality constraint of the form $\KL(p \| p_0) \leq \textrm{constant}$. 

We note that $\lambda \in [0,1]$ since $\lambda = \tau - \bar{C}p$, where $\tau \in [0,1]$ (because $\tau$ upper bounds the true expected cost) and $\bar{C}p \in [0,1]$ (recall that we assumed that costs are bounded between $0$ and $1$). In order to solve problem \eqref{opt:pac bayes final} to global optimality, we can thus simply search over the one-dimensional parameter $\lambda \in [0,1]$  (e.g., by simply discretizing the interval $[0,1]$, performing a bisection search, etc.) and find the setting of $\lambda$ that leads to the lowest optimal value for the corresponding REP.

\subsection{Continuously-Parameterized Control Policy Space}
\label{sec:continuous policy space}

We now consider policies $\pi_w$ parameterized by the vector $w \in \RR^d$ (e.g., neural networks parameterized by weights). We will consider stochastic policies defined by probability distributions over the parameters $w$. Here, we choose Gaussian distributions $w \sim \mathcal{N}(\mu, \Sigma)$ with diagonal covariance $\Sigma = \text{diag}(s)$ (with $s \in \RR_+^d$) and use the shorthand $\N_{\mu, s} := \N(\mu, \text{diag}(s))$. Using Gaussians makes computations easier since we can express the KL divergence between Gaussians in closed form (see Section \ref{sec:kl divergence}). We can then apply Algorithm \ref{a:pac bayes control} and choose $\mu, s$ to minimize the PAC-Bayes upper bound $C_{\textrm{PAC}}(\N_{\mu, s})$. In order to turn this into a practical algorithm, there are two primary issues we need to address. 

First, in order to minimize the bound $C_{\textrm{PAC}}(\N_{\mu, s})$, one would like to apply gradient-based methods (e.g., stochastic gradient descent). However, the cost function may not be a differentiable function of the parameters $w$. 
For example, in the case of designing obstacle avoidance policies, a natural (but non-differentiable) cost function is the one that assigns a cost of $1$ if the robot collides (and 0 otherwise).
To tackle this issue, we employ a differentiable surrogate for the cost function during optimization (note that the final bound is still evaluated for the original cost function). This surrogate will necessarily depend on the application at hand; we present examples in the contexts of obstacle avoidance and grasping in Section \ref{sec:examples}.

The second challenge is the fact that computing the training cost $C_S(\N_{\mu,s})$ requires computing the following expectation over policies:
\begin{equation}
\label{eq:contcase expectation over controllers}
\underset{w \sim \N_{\mu, s}}{\EE} [C(r_w; E)].
\end{equation}
For most realistic settings, this expectation cannot be computed in closed form. We address this issue in a manner similar to \citep{Dziugaite17}. In particular, in order to optimize $\mu$ and $s$ using gradient descent, we take gradient steps with respect to the following unbiased estimator of $C_S(\N_{\mu,s})$:
\begin{equation}
\label{eq:unbiased estimate}
\frac{1}{N} \sum_{E \in S} C(r_{\mu + \sqrt{s} \odot \xi} ; E), \quad \xi \sim  \N_{0, I_d}.
\end{equation}
In other words, in each gradient step we use an i.i.d. sample of $\xi$ and compute the gradient of \eqref{eq:unbiased estimate} with respect to $\mu$ and $s$.

At the end of the optimization procedure, we fix the optimal $\mu^\star$ and $s^\star$ and estimate the training cost $C_S(P) = C_S(\N_{\mu^\star,s^\star})$ by producing a large number of samples $w_1, \dots, w_L$ drawn from $\N_{\mu^\star, s^\star}$:
\begin{equation}
\hat{C}_S(\N_{\mu^\star,s^\star}) := \frac{1}{NL} \sum_{E \in S} \sum_{i=1}^L C(r_{w_i} ; E).
\end{equation}
We can then use a sample convergence bound (see \citep{Langford02}) to bound the error between $\hat{C}_S(\N_{\mu^\star, s^\star})$ and $C_S(\N_{\mu^\star, s^\star})$. In particular, the following bound is an application of the relative entropy version of the Chernoff bound for random variables (i.e., costs) bounded in $[0,1]$ and holds with probability $1 - \delta'$:
\begin{equation}
\label{eq:sample bound}
C_S(\N_{\mu^\star, s^\star}) \leq \bar{C}_S(\N_{\mu^\star, s^\star}; L, \delta') := \KL^{-1}(\hat{C}_S(\N_{\mu^\star, s^\star}) \| \frac{1}{L} \log(\frac{2}{\delta'})).
\end{equation}
Combining inequalities \eqref{eq:pac bayes kl} and \eqref{eq:sample bound} using the union bound, we see that the following bound holds with probability at least $1 - \delta - \delta'$:
\begin{equation}
\label{eq:contcase KL inv}
C_\D(\N_{\mu^\star,s^\star}) \leq C_{\textrm{bound}}^\star:= \KL^{-1} \Bigg(\bar{C}_S(\N_{\mu^\star,s^\star}; L, \delta') \| \frac{\KL(\N_{\mu^\star,s^\star} \| P_0) + \log(\frac{2 \sqrt{N}}{\delta})}{N}\Bigg).
\end{equation}
This is the final version of our bound on the expected performance of policies (drawn from $\N_{\mu^\star,s^\star}$). 

Algorithm \ref{a:pac bayes SGD} summarizes our approach from this section. Note that in order to ensure positivity of $s \in \RR_+^d$, we perform the optimization with respect to $\eta := \log(s)$. 

\begin{algorithm}[h]
	\algsetup{linenosize=\normalsize}
  \small
  \caption{{\small PAC-Bayes Policy Learning via Gradient Descent}}
  \label{a:pac bayes SGD}
  \begin{algorithmic}[1]
    \STATE {\bf Inputs:} 
    \STATE $S = \{E_1, \dots, E_N\}$: Training environments 
    \STATE $\delta, \delta' \in (0,1)$: Probability thresholds
    \STATE $P_0$: Prior over policies
    \STATE $\mu, s \in \RR^d$: Initializations for $\mu$ and $s$
    \STATE $\gamma$: step size for gradient descent
    
    \STATE {\bf Outputs:}  \\
    \STATE $\mu^\star, s^\star$: Optimal $\mu, s$
    \STATE $C_{\textrm{bound}}^\star:= \KL^{-1} \Big(\bar{C}_S(\N_{\mu^\star,s^\star}; L, \delta') \| \frac{\KL(\N_{\mu^\star,s^\star} \| P_0) + \log(\frac{2 \sqrt{N}}{\delta})}{N}\Big)$
    
    \STATE {\bf Procedure:}
    \STATE $B(\mu, s, w) := \frac{1}{N} \sum_{E \in S} C(r_w ; E) + \sqrt{\frac{\KL(\N_{\mu^\star,s^\star} \| P_0) + \log(\frac{2 \sqrt{N}}{\delta})}{2N}}$
    \WHILE{$\neg$converged}
    \STATE Sample $\xi \sim  \N_{0, I_d}$ and set $w \leftarrow \mu + \sqrt{s} \odot \xi$
    \STATE $\mu \leftarrow \mu - \gamma \nabla_\mu B(\mu, \exp(\eta), w)$
    \STATE $\eta \leftarrow \eta - \gamma \nabla_\eta B(\mu, \exp(\eta), w)$
    \STATE $s \leftarrow \exp(\eta)$
    \ENDWHILE
  \end{algorithmic}
\end{algorithm}

\section{Extensions}
\label{sec:extensions}

In this section, we present two extensions to the basic framework presented so far. In Section \ref{sec:stochastic rollouts}, we discuss extensions to systems with stochastic dynamics or sensor measurements. In Section \ref{sec:domain shift}, we present an approach that allows us to tackle settings where training and test environments are drawn from different distributions. 

\subsection{Stochastic Rollout Functions}
\label{sec:stochastic rollouts}

In our problem formulation in Section \ref{sec:problem formulation}, we assumed that the rollout function $r_w: \mathcal{E} \rightarrow (\mathcal{X} \times \mathcal{U})^H$ is deterministic (i.e., once the environment is fixed, the resulting state-action trajectory obtained by applying a given policy is completely determined). Here we briefly sketch an extension of our framework to settings where the rollout function is stochastic (e.g., due to stochasticity in the dynamics of the system or in sensor measurements). This is made possible by a reinterpretation of the variable $w$. Previously, $w$ corresponded to parameters of the control policy. Suppose now that we think of $w$ as consisting of two components $w \coloneqq [w_\text{int}, w_\text{ext}]$; an ``internal" component $w_\text{int}$ corresponding to parameters of the control policy (just as before), and an additional ``external" component corresponding to uncertain parameters (e.g., external disturbances that the robot might experience). The rollout function now has the following structure: $r_{[w_\text{int}, w_\text{ext}]}: \mathcal{E} \rightarrow (\mathcal{X} \times \mathcal{U})^H$. The stochasticity in $w_\text{int}$ is directly set by us (i.e., by choosing a prior $P_0$ and posterior $P$ as before). However, the stochasticity over $w_\text{ext}$ is beyond our control. 

We note that the structure of the resulting problem is identical to the original formulation considered in Section \ref{sec:problem formulation}. The only difference comes from the fact that a portion of the stochasticity in the rollouts is beyond our control. We can thus directly apply Theorem \ref{thm:pac bayes control} in order to obtain an upper bound on the true expected cost. In particular, let $P_0$ and $P$ be the prior and posterior over $w_\text{int}$ (as before) and suppose that the distribution over $w_\text{ext}$ is given by $P_\text{ext}$. Further, assume that $w_\text{int}$ and $w_\text{ext}$ are independent random variables. We can then define $P_0'$ and $P'$ to be the prior and posterior distributions over $w \coloneqq [w_\text{int}, w_\text{ext}]$ and evaluate the ``regularizer" term in the PAC-Bayes bound in Theorem \ref{thm:pac bayes control} by noting that:
\begin{align}
\KL(P' \| P_0') \ &=\revision{\EE_{P,P_\text{ext}}\Bigg{[}\log\bigg{(}\frac{P_\text{ext} P}{P_\text{ext} P_0}\bigg{)}\Bigg{]} =\EE_{P,P_\text{ext}}\Bigg{[}\log\bigg{(}\frac{P}{P_0} \bigg{)}\Bigg{]}} \label{eq:indep_assumption1} \\
&=\revision{\EE_{P}\Bigg{[}\log\frac{P}{P_0}\Bigg{]}\underbrace{\EE_{P_\text{ext}}[1]}_{=1}} = \KL(P \| P_0). \label{eq:indep_assumption2}
\end{align}  
\revision{The equality between the two lines} follows from the fact that $w_\text{int}$ and $w_\text{ext}$ are independent. In order to evaluate the training cost $C_S(P') = \frac{1}{N} \sum_{E \in S} \EE_{w \sim P'} [C(r_w; E)]$, we can employ the sampling procedure described in Section \ref{sec:continuous policy space} (i.e., sampling the disturbances $w_\text{ext} \sim P_\text{ext}$ in a manner analogous to how $w$ was sampled in Section \ref{sec:continuous policy space}). Thus, the framework for the deterministic rollout setting can be applied with almost no modifications in order to handle the stochastic rollout case (as long as one can sample disturbances $w_\text{ext}$ and assuming that the disturbances are drawn independently of $w_\text{int}$).

\subsection{Distributionally-Robust Control Policies}
\label{sec:domain shift}

So far, we have assumed that the robot will be tested on environments that are drawn from the same distribution as the training environments. We will now address the setting where this assumption is not valid and learn \emph{distributionally-robust policies} (i.e., policies that are robust to changes in the distribution from which environments are drawn). We will assume that the distribution $\D'$ from which test environments are drawn is bounded in terms of an $f$-\emph{divergence} (see below) from the training distribution $\D$ and formulate a robust version of the PAC-Bayes bound already described. 
\begin{definition}[$f$-divergence between $\D'$ and $\D$ \citep{Nguyen10}] For any convex $f(x)$ such that $f(1) = 0$, let
\begin{equation}
D_f(\D'||\D) := \underset{E\sim\D}{\EE}\Bigg{[}f\Bigg{(}\frac{\D'}{\D}\Bigg{)}\Bigg{]}.
\end{equation}
\end{definition}
The $f$-divergences encapsulate a broad class of divergences between distributions and include the KL divergence as a special case (with $f(x) = x\log x$). 
We will assume that the test distribution $\D'$ is bounded in terms of an $f$-divergence: $D_f(\D'||\D) \leq \mathcal{B}$ (but no further assumption on $\D'$ will be made). The control policy we learn will have an associated guarantee on \emph{any} test distribution that satisfies this assumption. 

\begin{theorem}[$f$-divergence between $\D'$ and $\D$ in terms of $f$ and $f^*$ \citep{Nguyen10}]
	 For a given $f(x)$ and its convex conjugate $f^*(y) \coloneqq \sup_{x \in \mathbb{R}} \big{[}xy - f(x)\big{]}$, we can write the $f$-divergence $D_f(\D'||\D)$ in terms of only $f$ and its conjugate:
\begin{equation}
D_f(\D'||\D) = \sup_{C:\Pi \times \mathcal{E} \rightarrow \mathbb{R}}\bigg{(}\underset{E \sim \D'}{\EE} \ \underset{w \sim P}{\EE} [C(r_w; E)] 
 - \underset{E \sim \D}{\EE} \ \underset{w \sim P}{\EE} [f^*(C(r_w; E))]\bigg{)}.
\end{equation}
\end{theorem}

The supremum above is taken over all functions $C$ that result in the expectations in the RHS being finite. Thus, for any particular (cost) function $C$, we obtain a lower bound on the supremum term. This allows us to obtain the following useful corollary. 
\begin{corollary} [$f$-Divergence variational inequality] If $D_f(\D'||\D) \leq \mathcal{B}$, then
\label{thm:fdiv var ineq}
\begin{equation}
C_{\D'}(P) := \underset{E \sim \D'}{\EE} \ \underset{w \sim P}{\EE} [C(r_w; E)] \leq \mathcal{B} + \underset{E \sim \D}{\EE} \ \underset{w \sim P}{\EE} [f^*(C(r_w; E))].
\end{equation}
\end{corollary}

Note that this corollary is valid for any $f$-divergence. In particular, it holds when $f(x) = x\log x$, making $f^*(y) = e^{y-1}$. With this choice of $f$, we obtain an upper bound on $C_{\D'}(P)$ in terms of the bound $\mathcal{B}$ on the KL divergence. 

\begin{corollary} [KL divergence variational inequality]
\label{cor:kl var ineq}
For $f$-divergence with $f(x) = x\log x$ and $D_f(\D'||\D) \leq \mathcal{B}$, we have $D_f(\D'||\D) = \KL(\D'||\D) \leq \mathcal{B}$ and 

\begin{equation}
\label{eq:kl var ineq}
C_{\D'}(P) \leq \mathcal{B} + \underset{E \sim \D}{\EE} \ \underset{w \sim P}{\EE} [e^{C(r_w; E)}] - 1.
\end{equation}
\end{corollary}
While Corollary \ref{cor:kl var ineq} provides a valid inequality in the special case of the KL divergence, a tighter bound can be obtained using the \emph{Donsker-Varadhan (DV)} inequality.

\begin{theorem}[Donsker-Varadhan variational inequality \citep{Donsker75}; Theorem 3.2 in \citep{Gray11}]
\label{thm:dv var ineq}
If $\KL(\D'||\D) \leq \mathcal{B}$, then
\begin{equation}
\label{eq:DK-ineq bound} 
C_{\D'}(P) \leq \mathcal{B} + \log \Big{(}\underset{E \sim \D}{\EE} \ \underset{w \sim P}{\EE} [e^{C(r_w; E)}]\Big{)}.
\end{equation}
\end{theorem}

The DV inequality provides a tighter bound than inequality \eqref{eq:kl var ineq} since $x-1 \geq \log(x), \ \forall x > 0$. For the rest of this section, we will specialize our discussion to the KL divergence and use the DV inequality. However, we note that our approach generalizes to any $f$-Divergence by leveraging Corollary \ref{thm:fdiv var ineq}. 

As written, inequality \eqref{eq:DK-ineq bound} cannot be used to directly upper bound $C_{\D'}(P)$ since \linebreak $\EE_{E \sim \D} \ \EE_{w \sim P} [e^{C(r_w; E)}]$ is not an observable quantity. However, we can leverage the inequality \eqref{eq:pac inequality control} to obtain an upper bound on $C_{\D'}(P)$ in terms of observable quantities. Since Theorem \ref{thm:pac bayes control} holds for any cost function between $0$ and $1$, we will be able to apply inequality \eqref{eq:pac inequality control} if we replace the cost with an exponentiated one, as long as we rescale to stay between 0 and 1. Thus, if we make the substitution 
$$C(r_w; E) \leftarrow \frac{e^{C(r_w; E)}-1}{e-1},$$
we obtain the following bound using inequality \eqref{eq:pac inequality control}:
\begin{equation}
\label{eq:transformed pac-bayes bound}
\underset{E \sim \D}{\EE} \ \underset{w \sim P}{\EE} [e^{C(r_w; E)}] \leq \frac{1}{N} \sum_{E \in S} \underset{w \sim P}{\EE} [e^{C(r_w; E)}] + (e-1)\sqrt{\frac{\KL(P \| P_0) + \log(\frac{2 \sqrt{N}}{\delta})}{2N}}.
\end{equation}
This inequality holds because the transformation keeps the cost between $[0,1]$. 
Now, since we assumed that $\KL(\D'||\D) \leq \mathcal{B}$, we can apply Theorem \ref{thm:dv var ineq} to bound $C_{\D'}(P)$.

\begin{corollary}[Distributionally-robust PAC-Bayes bound]
\label{cor:robust PAC-Bayes bound}
For any $\D'$ such that $\KL(\D'||D) \leq \mathcal{B}$ and any $\delta \in (0,1)$, with probability at least $1 - \delta$ over sampled environments $S \sim \D^N$ the following inequality holds:
\begin{equation}
\label{eq:cor robust PAC-Bayes bound}
C_{\D'}(P) \leq \mathcal{B} + \log \Bigg{(} \frac{1}{N} \sum_{E \in S} \underset{w \sim P}{\EE} [e^{C(r_w; E)}] + (e-1)\sqrt{\frac{\KL(P \| P_0) + \log(\frac{2 \sqrt{N}}{\delta})}{2N}} \Bigg{)}.
\end{equation}
\end{corollary}

The RHS of inequality \eqref{eq:cor robust PAC-Bayes bound} gives us an upper bound $C_{\text{PAC}'}$  on $C_{\D'}$. We can thus apply an analogous procedure to Algorithm \ref{a:pac bayes control} to obtain $P^\star_{\text{PAC}'}$ (a distributionally-robust stochastic policy) by minimizing this upper bound and $C^\star_{\text{bound}'}$ (the final distributionally-robust PAC-Bayes bound). 

In the finite policy space setting, we can apply a procedure similar to the one employed in Section \ref{sec:finite policy space} to write an REP that minimizes the bound $C_{\text{PAC}'}$. Define:
\begin{equation}
	\hat{C}_e[i,j] \coloneqq e^{C(r_{\pi_j}; E_i)}.
\end{equation}
We then have
\begin{equation}
\frac{1}{N} \sum_{E \in S} \underset{\pi \sim P}{\EE} [e^{C(r_\pi; E)}] = \frac{1}{N} \sum_{i=1}^N \sum_{j=1}^L \hat{C}_e[i,j] p[j] := \bar{C}_ep.
\end{equation}
We can then minimize the bound $C_{\text{PAC}'}$ using an REP analogous to Problem \eqref{opt:pac bayes final}: 
\begin{flalign}
		& & \underset{p \in \RR^L, \tau, \lambda}{\textrm{min}} \hspace*{1cm} & \tau &&
		  \\
		& & \text{s.t.} \hspace*{1cm} & \lambda^2 \geq (e-1)^2\frac{\KL(p \| p_0) + \log(\frac{2 \sqrt{N}}{\delta})}{2N} && \nonumber \\
		& & & \lambda = \tau - \bar{C}_ep, \ \lambda \geq 0 && \nonumber \\
		& & &  0 \leq p \leq 1, \ \sum_j p[j] = 1. \nonumber 
\end{flalign}
Here $\bar{C}_ep \in [1, e]$, and since $\tau$ upper bounds the true expected cost, we are only interested in values of $\tau \in [1, e]$. Thus an optimal $\lambda$ can be found by searching over $\lambda \in [0, e-1]$, which can then be used to obtain $P_{\text{PAC}'}^\star$ and $C_{\text{PAC}'}$. Additionally, in the continuously-parameterized control policy space case, we can make modifications to Algorithm \ref{a:pac bayes SGD} and equations (\ref{eq:contcase expectation over controllers}$-$\ref{eq:contcase KL inv}) to directly adapt the SGD approach to minimize $C_{\text{PAC}'}$. 

Finally, as in Section \ref{sec:computing pac-bayes controllers}, the final distributionally-robust upper bound  $C^\star_{\text{bound}'}$ is not computed as $C_{\text{PAC}'}(P^\star_{\text{PAC}'})$ but with an analogue to the KL inverse in equation \eqref{eq:pac bayes learning kl version}:
\begin{flalign}
		& & \underset{c_{\D'}, c_\D \in [0,1]}{\textrm{max}} \hspace*{1cm} & c_{\D'} && \\
		& & \text{s.t.} \hspace*{1cm} & \KL(C_S(P)||c_\D) \leq \frac{\KL(P \| P_0) + \log(\frac{2 \sqrt{N}}{\delta})}{N} && \nonumber \\
		& & & \KL(c_{\D'}||c_\D) \leq \mathcal{B}.  \nonumber
\end{flalign}
The first constraint is the same as in the non-robust case, and the second constraint accounts for the difference in the training and test distributions. Together these create an REP that can be solved to find $C^\star_{\text{bound}'}$.


\section{Examples}
\label{sec:examples}

In this section, we demonstrate our framework in simulation on two domains: obstacle avoidance (Section \ref{sec:reactive obstacle avoidance control}) and grasping (Section \ref{sec:grasping}). Our goal is to demonstrate the ability of our approach to learn control policies with strong guarantees on generalization to novel environments. We will consider a hardware example in Section \ref{sec:hardware}.

\subsection{Reactive Obstacle Avoidance Control}
\label{sec:reactive obstacle avoidance control}
In this section, we apply our approach on the problem of learning reactive obstacle avoidance policies for a ground vehicle model equipped with a depth sensor. We first consider a finite policy space $\Pi$ and leverage the REP-based framework described in Section \ref{sec:finite policy space}. We then consider continuously parameterized policies and apply the approach from Section \ref{sec:continuous policy space}. Finally, we apply the approach from Section \ref{sec:domain shift} to learn distributionally-robust control policies. 

{\bf Dynamics.} A pictorial depiction of the ground vehicle model is provided in Figure \ref{fig:husky}. The state of the system is given by $[x,y,\psi]$, where $x$ and $y$ are the x and y positions of the vehicle respectively, and $\psi$ is the yaw angle. We model the system as a differential drive vehicle with the following nonlinear dynamics:
\begin{equation}
\left[ \begin{array}{c}
\dot{x} \\
\dot{y} \\
\dot{\psi} \end{array} \right] = 
 \left[ \begin{array}{c} 
-\frac{r}{2}(u_l + u_r)\sin(\psi) \\
\frac{r}{2}(u_l + u_r)\cos(\psi) \\
\frac{r}{L}(u_r - u_l) \end{array} \right],
\label{eq:husky dynamics}
\end{equation}
where $u_l$ and $u_r$ are the control inputs (corresponding to the left and right wheel speeds respectively), $r = 0.1$m corresponds to the radius of the wheels, and $L = 0.5$m corresponds to the width of the base of the vehicle. We set:
\begin{equation}
u_l = u_0 - u_{\text{diff}}, \quad u_r = u_0 + u_{\text{diff}},
\end{equation}
where $u_0 = v_0/r$ with $v_0 = 2.5$m/s. This ensures that the robot has a fixed speed $v_0$. We limit the turning rate by constraining $u_{\text{diff}} \in [-u_0/2, u_0/2]$. The system is simulated as a discrete-time system with time-step $\Delta t = 0.05$s. 
 
{\bf Obstacle environments.} A typical obstacle environment is shown in Figure \ref{fig:husky} and consists of $N_{\text{obs}}$ cylinders of varying radii along with three walls that bound the environment between $x \in [-5,5]$m and $y \in [0,10]$m. Environments are generated by first sampling the integer $N_{\text{obs}}$ uniformly between $20$ and $40$, and then independently sampling the x-y positions of the cylinders from a uniform distribution over the ranges $x \in [-5,5]$m and $y \in [2,10]$m. The radius of each obstacle is sampled independently from a uniform distribution over the range $[0.05, 0.2]$m. The robot's state is always initialized at $[x,y,\psi] = [0,1,0]$.

{\bf Obstacle Avoidance Policies.} We assume that the robot is equipped with a depth sensor that provides distances $y[i]$ along $20$ rays in the range $\theta[i] \in [-\pi/3, \pi/3]$ radians (positive is clockwise) up to a sensing horizon of $5$m (as shown in Figure \ref{fig:husky}). A given sensor measurement $y$ thus belongs to the space $\Y = \RR^{20}$. Let $\hat{y} = 1/y \in \RR^{20}$ be the inverse distance vector computed by taking an element-wise reciprocal of $y$. We then choose $u_{\text{diff}}$ as the following dot product:
\begin{equation}
\label{eq:controller}
u_{\text{diff}} = K \cdot \hat{y}. 
\end{equation}
An example of $K \in \RR^{20}$ is:
\begin{equation}
\label{eq:K}
K[i] = 
\begin{cases} 
      (y_0/x_0)(x_0 - \theta[i]) & \text{if} \ \theta[i] \geq 0, \\
      (y_0/x_0)(-x_0 - \theta[i]) & \text{if} \ \theta[i] < 0.
   \end{cases}
\end{equation}
Such a $K$ is shown in Figure \ref{fig:K_vs_theta}. For $\theta[i] > 0$, $K[i]$ is a linear function of $\theta[i]$ with x- and y-intercepts equal to $x_0$ and $y_0$ respectively. This linear function is reflected about the origin for $\theta[i] < 0$. 

\begin{figure}[h]
\begin{center}
\includegraphics[width=0.8\columnwidth]{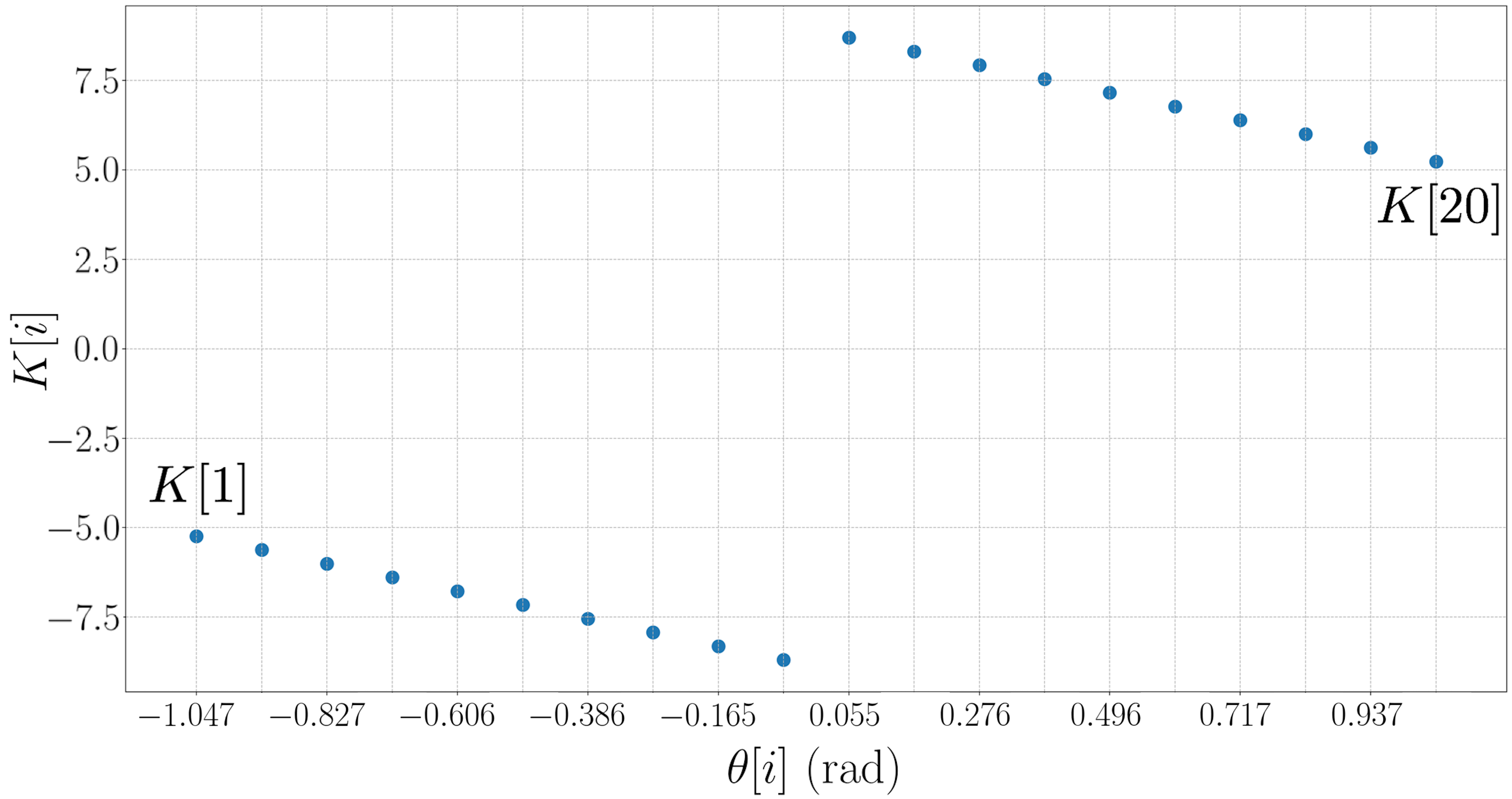}
\end{center} 
\caption{\footnotesize{Example of $K[i]$ as a function of $\theta[i]$.}  \label{fig:K_vs_theta}}
\end{figure}

Intuitively, this corresponds to a simple reactive policy that computes a weighted combination of inverse distances in order to turn away from obstacles that are close. As a simple example, consider the case where we have two obstacles: one located $4$m away along $\theta = -\pi/4$ (i.e., to the robot's left) and the other located $1$m away along $\theta = \pi/4$ (i.e., to the robot's right). The computed control input will then be $u_{\text{diff}} > 0$ (i.e., robot turns left) since the inverse depth for the obstacle to the right is larger than that of the obstacle to the left. Simple reactive policies of this kind have been shown to be quite effective in practice \citep{Arkin98,Beyeler09,Ross13,Conroy09}, but can often be challenging to tune by hand in order to achieve good expected performance across \emph{all} environments. We tackle this challenge by applying the PAC-Bayes control framework proposed here. 

{\bf Results (finite policy space).} In order to obtain a finite policy space, we choose $L = 50$ different $K$'s of the form \eqref{eq:K} by choosing different x and y intercepts $x_0$ and $y_0$. In particular, $(x_0, y_0)$ is chosen by discretizing the space $[0.1,5.0] \times [0,10.0]$ into 5 values for $x_0$ and 10 values for $y_0$. Our control policy space is thus $\Pi = \{\pi_1, \dots, \pi_L\}$, where each policy $\pi_i$ corresponds to a particular choice of $K$. 

We consider a time horizon of $T = 100$ and assign a cost of $1$ if the robot collides with an obstacle during this period and a cost of $0$ otherwise. We choose a uniform prior over the policy space $\Pi$ and apply the REP framework from Section \ref{sec:finite policy space} in order to optimize a distribution over policies. 
The PyBullet package \citep{Coumans18} is used to simulate the dynamics and depth sensor; we use these simulations to compute the elements of the cost matrix $\bar{C}$ (ref. Section \ref{sec:finite policy space}). Each simulation takes $\sim0.01$s to execute in our implementation (note that the computation of the different elements of $\bar{C}$ can be entirely parallelized). Given the matrix $\bar{C}$ with 100 sampled environments, each REP (corresponding to a fixed value of $\lambda$ in Problem \eqref{opt:pac bayes final}) takes $\sim0.05$s to solve using the CVXPY package \citep{CVXPY} and the SCS solver \citep{SCS}. We discretize the interval $[0,1]$ into 100 values to find the optimal $\lambda$. Complete code for this implementation is freely available on GitHub\footnote{Code: \href{https://github.com/irom-lab/PAC-Bayes-Control}{https://github.com/irom-lab/PAC-Bayes-Control}}.

Table \ref{tab:results obstacle avoidance} presents the upper bound $C^\star_{\text{bound}}$ on the true expected cost of the PAC-Bayes control policy $P^\star_\text{PAC}$ (ref. Algorithm \ref{a:pac bayes control}) for different sample sizes $N$ with $\delta = 0.01$. The table also presents an estimate of the true expected cost $C_\D(P^\star_\text{PAC})$ obtained by sampling $10^5$ environments. As the table illustrates, the PAC-Bayes bound provides strong guarantees even for relatively small sample sizes. For example, using only $100$ samples, the PAC-Bayes policy is guaranteed (with probability $1-\delta=0.99$) to have an expected success rate of $82.2\%$ (i.e., an expected cost of $0.178$). Exhaustive sampling indicates that the expected success rate for the PAC-Bayes policy is approximately $91.3\%$ for this case. Videos of representative trials on test environments can be found at \href{https://youtu.be/y4zTK79s1mI}{https://youtu.be/y4zTK79s1mI}. 
  
  \begin{table}[t!]
\small
\begin{center}
  \begin{tabular}{ | l | c | c | c | c | }
    \hline 
    N ($\#$ of training environments) & 100 & 500 & 1000 & 10000  \\ \hline
    \hline
    PAC-Bayes bound ($C^\star_{\text{bound}}$) &  0.178  & 0.135  & 0.121 & 0.096 \\ \hline
    True expected cost (estimate) &  0.087  & 0.084  & 0.088 & 0.083 \\
    \hline
  \end{tabular}
  \vspace{5pt}
    \caption[]{\footnotesize{Comparison of PAC-Bayes bound with the true expected cost (estimated by sampling $10^5$ obstacle environments). Using only 100 samples, with probability $0.99$ over samples, the PAC-Bayes policy is guaranteed to have an expected success rate of $82.2\%$. The true expected success rate is approximately $91.3\%$.}  \label{tab:results obstacle avoidance}}
  \end{center}
  \vspace{-10pt}
  \end{table}

{\bf Results (continuous policy space).} Next, we consider a continuously parameterized policy space $\Pi$ and apply the approach described in Section \ref{sec:continuous policy space}. In particular, we parameterize our policy using the matrix $K \in \RR^{20}$ in equation \eqref{eq:controller} while ensuring symmetry of the control law, i.e., we constrain $K[i] = -K[j]$ for $\theta[i] = -\theta[j]$ (note that $K$ is no longer constrained to have the linear form from equation \eqref{eq:K}). The dimensionality of the parameter space is thus $d = 10$. We apply Algorithm \ref{a:pac bayes SGD} to optimize a distribution $\N_{\mu^\star, s^\star}$ over policies. For the purpose of optimization, we employ a continuous surrogate cost function in place of the discontinuous 0-1 cost. We choose this to be the negative of the minimum distance to an obstacle along a trajectory (appropriately scaled to lie within $[0,1]$). Note that we employ this surrogate cost only for optimization; all results are presented for the 0-1 cost. Gradients in Algorithm \ref{a:pac bayes SGD} are estimated numerically. We choose a prior $P_0 = \N_{\mu_0, s_0}$ with $s_0 = 0.01$; the mean $\mu_0$ is given by a vector $K$ of the form \eqref{eq:K} with x-intercept $2.5$ and y-intercept $10.0$. 

We use $N = 100$ training environments and choose confidence parameters $\delta = 0.009$, $\delta' = 0.001$, and $L = 30,000$ samples to evaluate the sample convergence bound in equation \eqref{eq:sample bound}. Figure \ref{fig:K_vs_theta_continuous} shows the mean $\mu^\star$ of the optimized policy obtained using Algorithm \ref{a:pac bayes SGD}. 
The corresponding PAC-Bayes bound $C^\star_{\text{bound}}$ is $0.224$. Thus, with probability $0.99$ over sampled training data, the optimized PAC-Bayes policy is guaranteed to have an expected success rate of $77.6\%$. Exhaustive sampling with $10^5$ environments indicates that the expected success rate is approximately $92.5\%$. Videos of representative trials on test environments can be found at \href{https://youtu.be/y4zTK79s1mI}{https://youtu.be/y4zTK79s1mI}.

\begin{figure}[t]
\begin{center}
\includegraphics[width=0.8\columnwidth]{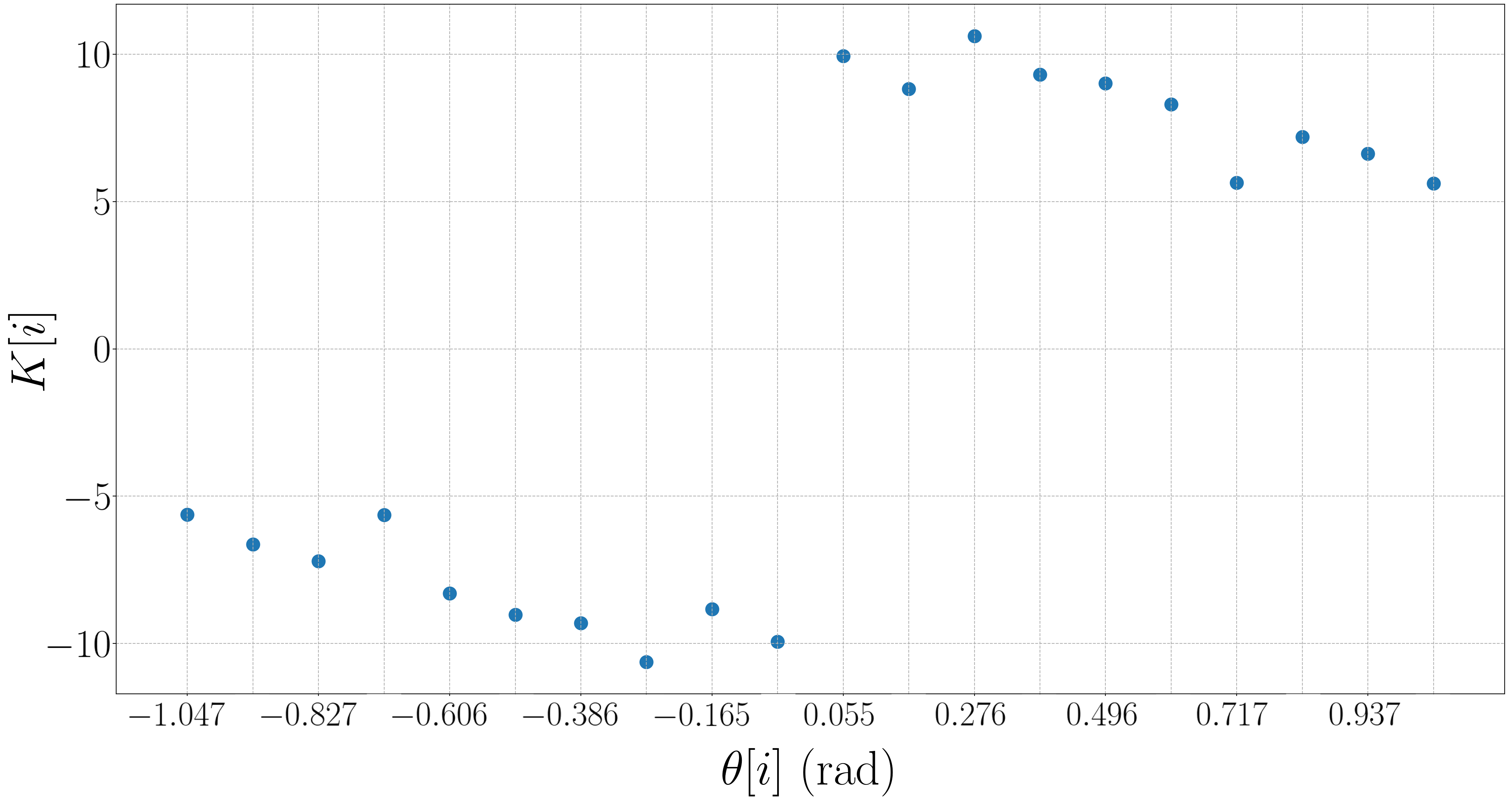}
\end{center} 
\caption{\footnotesize{Optimized $K$ corresponding to $\mu^\star$.}  \label{fig:K_vs_theta_continuous}}
\end{figure}

{\bf Results (distributionally-robust policies).} We now apply the approach presented in Section \ref{sec:domain shift} to learn distributionally-robust policies. 
Complete code for the implementation of the example here is freely available on GitHub\footnote{Code: \href{https://github.com/irom-lab/PAC-Bayes-Control/tree/master/Extension-Domain_Shifts}{https://github.com/irom-lab/PAC-Bayes-Control/tree/master/Extension-Domain\_Shifts}}.
 To provide a concrete way of bounding $\KL(\D'||\D)$, the training and test distributions differ only in the way that the radius of the cylindrical obstacles for that environment is sampled. For a single environment, all obstacles will have the same radius, but the beta distribution from which this radius is sampled differ. This means that $\KL(\D'||\D) = \KL(\mathbb{B}(\alpha', \beta')||\mathbb{B}(\alpha, \beta))$ where $\mathbb{B}(\alpha,\beta)$ is the beta distribution, with parameters $\alpha$ and $\beta$, used to sample the radius:

\begin{equation}
	\KL(\mathbb{B}(\alpha', \beta')||\mathbb{B}(\alpha, \beta)) = \log \bigg{(}\frac{\text{B}(\alpha,\beta)}{\text{B}(\alpha',\beta')}\bigg{)} + (\alpha' -\alpha)\psi(\alpha') + (\beta'-\beta)\psi(\beta') + (\alpha - \alpha' + \beta -\beta')\psi(\alpha' + \beta')
\end{equation}
where $\alpha$ and $\beta$ are the beta distribution parameters for $\D$, $\alpha'$ and $\beta'$ are the beta distribution parameters for $\D'$, $\text{B}(\cdot,\cdot)$ is the beta function (distinct from the beta distribution), and $\psi(\cdot)$ is the digamma function. This divergence can be computed analytically with a symbolic integrator such as Mathematica \citep{MATHEMATICA}. See Figure \ref{fig:pdf-beta} for the probability density functions of the distributions used to determine the radii of obstacles for this example, where $\KL(\D'||\D) \leq \mathcal{B} = 0.0819$. Note that for any test distribution over environments that satisfies the inequality $\KL(\D'||\D) \leq 0.0819$, the computed bound $C^\star_{\text{bound}'}$ will be valid.  

\begin{figure}[h]
 \centering
 \includegraphics[width=0.6\textwidth]{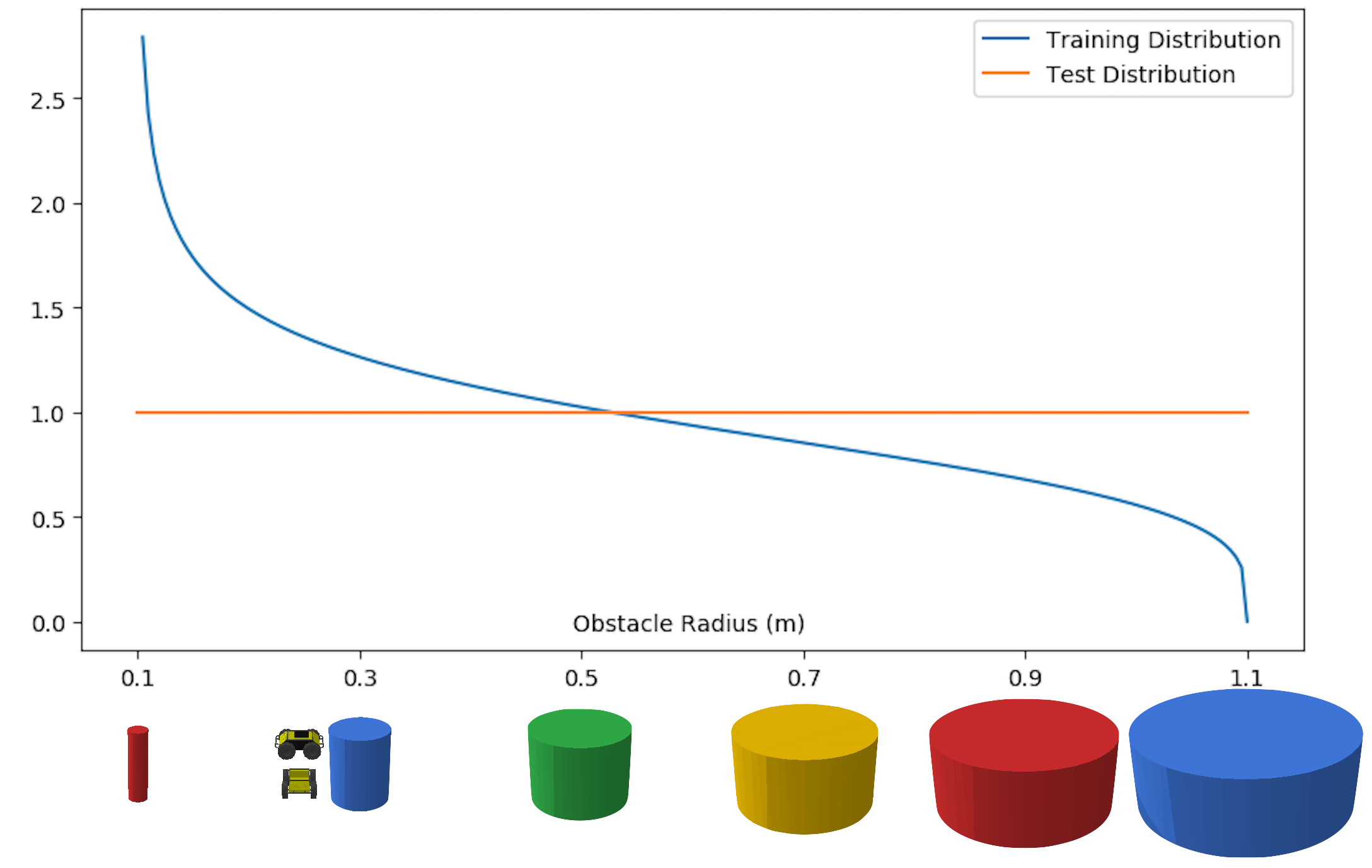}
 \caption{\footnotesize{Probability density functions (PDFs) for the beta distributions used to determine obstacle radii for the training and test environments. Here $\alpha = 0.8$, $\beta = 1.25$, $\alpha' = 1$, and $\beta' = 1$ which makes $\KL(\D'||\D) = \KL(\mathbb{B}(\alpha', \beta')||\mathbb{B}(\alpha, \beta)) \leq \mathcal{B} = 0.0819$. The robot's radius is $0.27$m, depicted next to the $0.3$m obstacle, and the radii of the obstacles are bounded between $0.10$m and $1.10$m.}}
 \label{fig:pdf-beta}
\end{figure}
\begin{figure}[h!]
 \centering
 \subfigure{\includegraphics[width=0.7\textwidth]{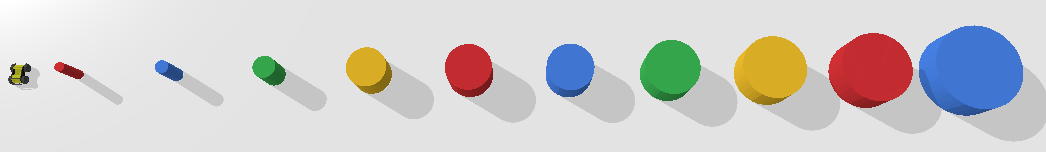}\label{fig:trenv}}\\
 \vspace{-5pt}
 \hspace{-1pt}
 \subfigure{\includegraphics[width=0.7\textwidth]{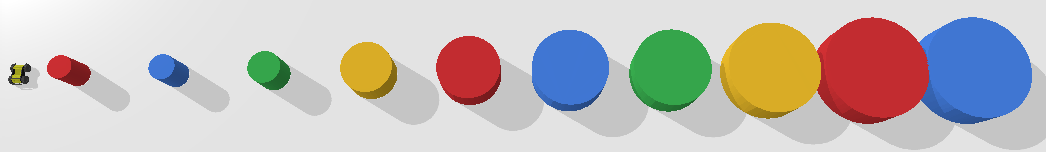}\label{fig:tenv}}
 \caption{\footnotesize{Comparison of obstacles generated by the beta distribution on the radius for the training (top image) and test (bottom image) environments. When the obstacles are sorted, it is easier to see that the training environment's obstacles are skewed towards a smaller radius, making those environments easier to navigate. This is apparent in the PDF comparison displayed in Figure \ref{fig:pdf-beta} as well.}}

 \label{fig:train/test envs avg radius}
\end{figure}

In Figure \ref{fig:train/test envs avg radius}, training and test obstacles are contrasted. Since the training environments are generated with a beta distribution that favors smaller radii, they are likely to be smaller than those generated with the test beta distribution (uniform distribution over the radius range). Additionally, the average obstacle radius, which can be calculated with the beta distribution parameters for the training and test distributions ($r_{\text{avg}} = \alpha/(\alpha + \beta)\times (r_{\text{max}} - r_{\text{min}}) + r_{\text{min}}$), differ by $0.11$m. This corresponds to 41\% of the robot's radius. Results on this example for the approach presented in Section \ref{sec:domain shift} are presented in Table \ref{table:robust PAC-bayes}. The table demonstrates that we are able to obtain strong bounds on generalization even in this distributionally-robust setting (albeit with a larger number of training environments). For example, with $5000$ training environments, we obtain a guaranteed expected success rate of $80.3\%$. We emphasize that this bound holds for \emph{any} $\D'$ that satisfies the constraint on the KL divergence (not just the specific test distribution chosen here). The estimated true success rate is approximately 91.8\% for $N = 5000$. With a sufficiently large number of training environments ($10^5$ in this example), the robust PAC-Bayes bound $\approx$ true expected cost on test + $\mathcal{B}$ + the scaled regularizer that appears in equation \eqref{eq:cor robust PAC-Bayes bound}.

\begin{table}[h!]
\begin{center}
\addtolength{\leftskip}{-1cm}
\addtolength{\rightskip}{-1cm}
\footnotesize
\begin{tabular}{ | p{10.05cm} | c | c | c | c | c | } \hline
N ($\#$ of training environments)     & 100   & 500   & 1000  & 5000  & 10000 \\ \hline\hline 

Robust PAC-Bayes bound ($C^\star_{\text{bound}'}$) & 0.453 & 0.276 & 0.238 & 0.197 & 0.185 \\ \hline
True (estimated) cost on $\D'$ using robust policy learned using $\D$ & 0.079 & 0.081 & 0.080 & 0.082 & 0.080 \\ \hline \hline
Non-robust PAC-Bayes bound on $\D$                 & 0.221 & 0.107 & 0.089 & 0.070 & 0.066 \\ \hline 
True (estimated) cost on $\D$ using policy learned on $\D$ & 0.054 & 0.057 & 0.054 & 0.054 & 0.056 \\ \hline 
Non-robust PAC-Bayes bound on $\D'$                  & 0.262 & 0.170 & 0.138 & 0.110 & 0.096 \\ \hline 
True (estimated) cost on $\D'$ using policy learned on $\D'$ & 0.081 & 0.080 & 0.081 & 0.079 & 0.083 \\ \hline 

\end{tabular}
\caption[]{\footnotesize{Comparison of distributionally-robust PAC-Bayes bounds with true costs estimated using $10^5$ environments. We are able to obtain strong bounds on generalization (albeit with a larger number of training environments than in the non-robust case). For example, with $5000$ training environments, we obtain a guaranteed expected success rate of $80.3\%$. The estimated true success rate is approximately 91.8\%. We also provide bounds and estimated true costs obtained using the standard (non-robust) PAC-Bayes framework as points of comparison. }
}
\label{table:robust PAC-bayes}
\end{center}
\end{table}

 \subsection{Grasping}
 \label{sec:grasping}

We now consider the problem of learning neural network-based grasping policies with guarantees on performance across novel objects. 

{\bf Dynamics and sensors.} The system we consider is shown in Figure \ref{fig:arm} and consists of a KUKA iiwa arm grasping an object placed on a table. The robot is equipped with a camera that provides RGB-D images. The entire simulation (rigid-body dynamics and sensing) is performed using the PyBullet simulator \citep{Coumans18}. 

{\bf Objects.} We use the ShapeNet database \citep{shapenet} to generate objects for grasping. ShapeNet consists of more than $50,000$ objects and thus provides a rich and challenging dataset. We scale the objects so they fit in a $10 \ \text{cm}^3$ volume. The masses of the objects are randomly chosen uniformly from the range $[0.05, 0.15]$ kg and the inertia matrices are randomly chosen diagonal matrices with elements chosen uniformly from the range $[0.75, 1.25]$. Objects are initialized in the environment by dropping them from a certain height above the table and allowed to settle. The initial orientation from which they are dropped is also randomized (yaw $\sim \mathcal{N}(0,0.5^2)$, roll $\sim \mathcal{N}(0,0.5^2)$, pitch $\sim \mathcal{N}(0,0.01^2)$). Note that the randomization for the initial pitch angle is smaller; this ensures that objects land ``upright" on the table. We randomly select $N = 2000$ objects as our training data. Figure \ref{fig:shapenet} shows randomly chosen representative objects from the ShapeNet database. 

\begin{figure}[h!]
 \centering
 
   \subfigure
   {\includegraphics[trim={87cm 40cm 68cm 40cm},clip,width=0.18\textwidth]{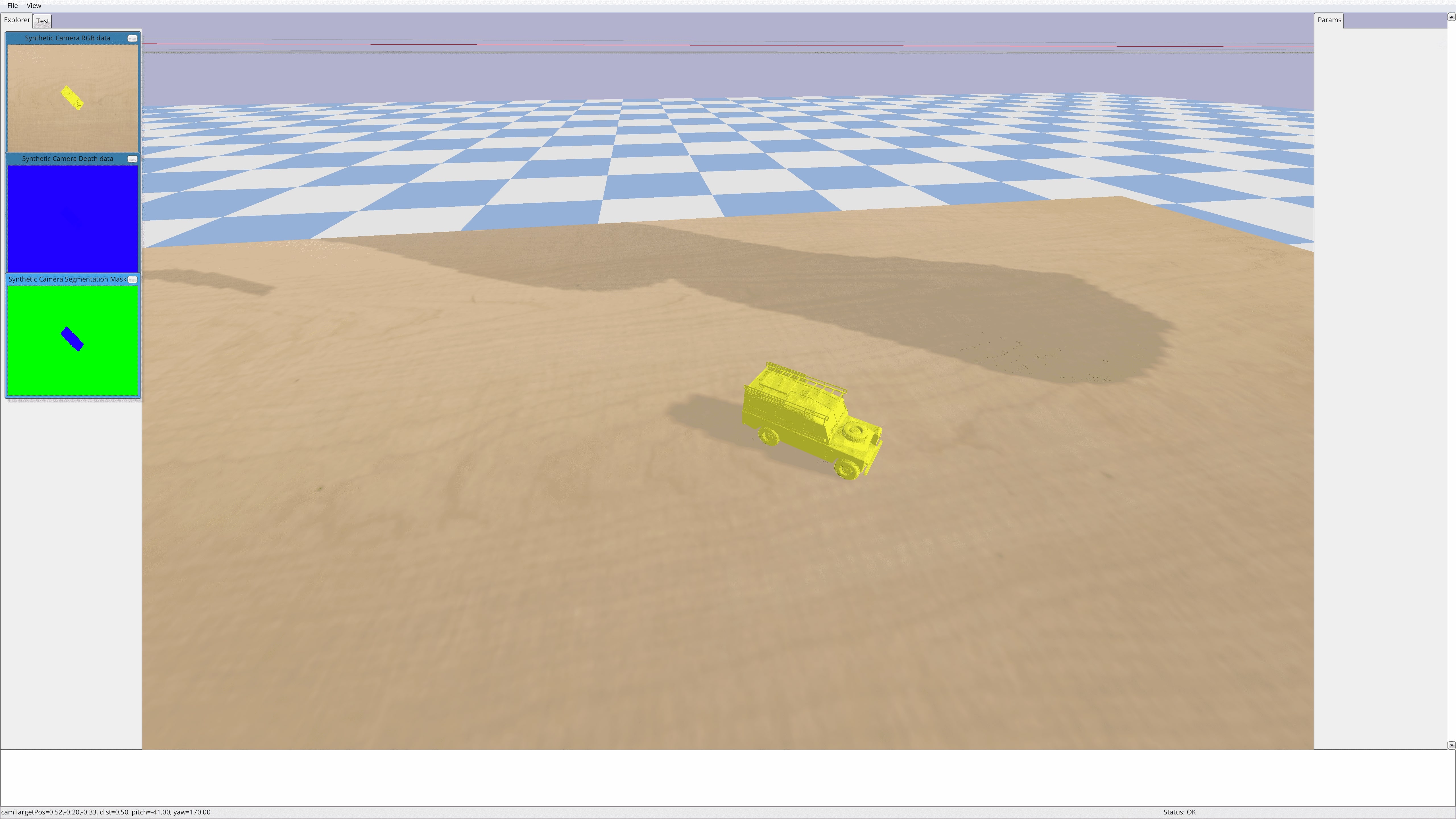}}
   \subfigure
   {\includegraphics[trim={90cm 40cm 65cm 40cm},clip,width=0.18\textwidth]{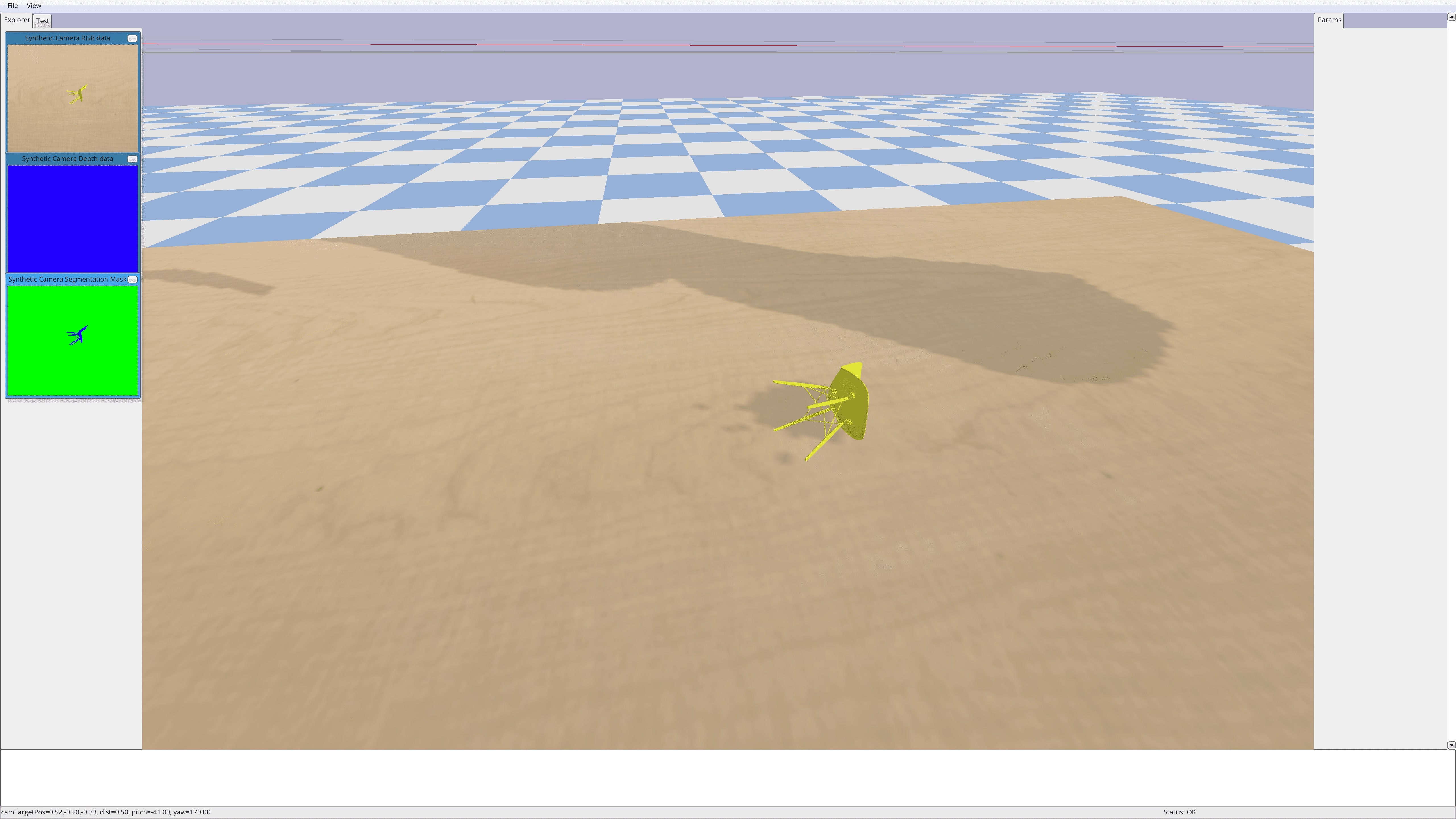}}
   \subfigure
   {\includegraphics[trim={90cm 40cm 65cm 40cm},clip,width=0.18\textwidth]{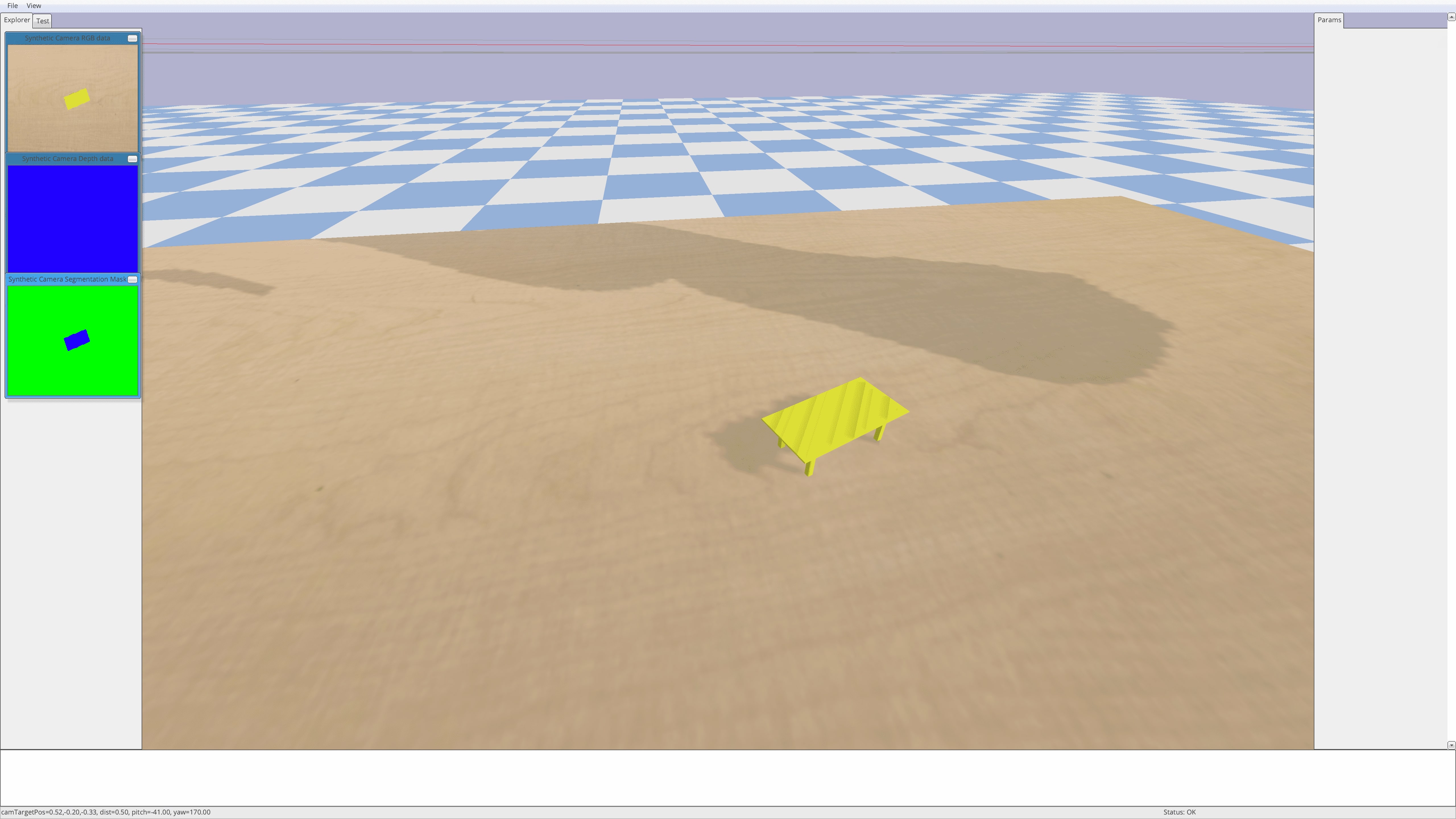}}
      \subfigure
   {\includegraphics[trim={90cm 40cm 65cm 40cm},clip,width=0.18\textwidth]{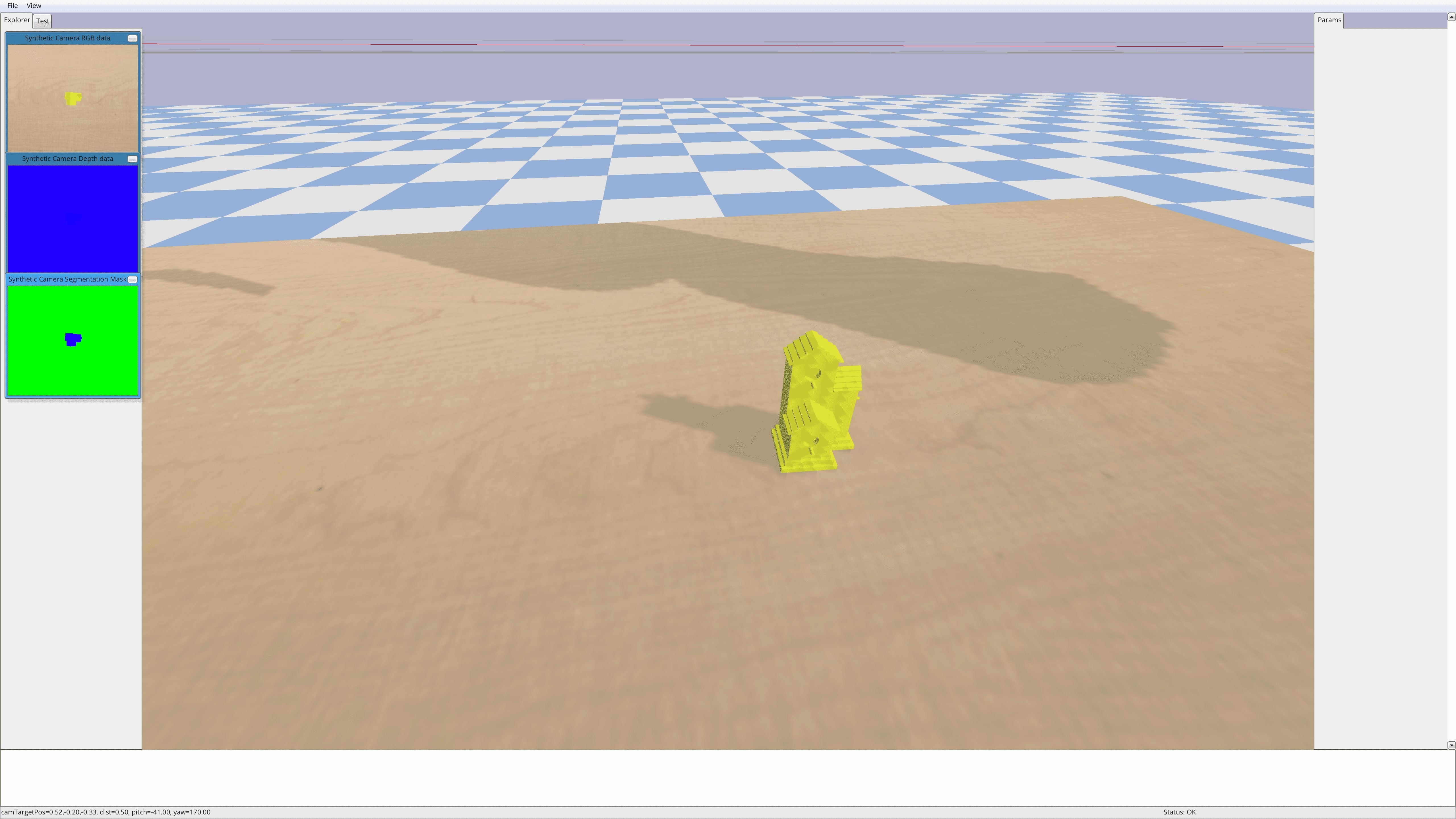}}
      {\includegraphics[trim={87cm 35cm 68cm 45cm},clip,width=0.18\textwidth]{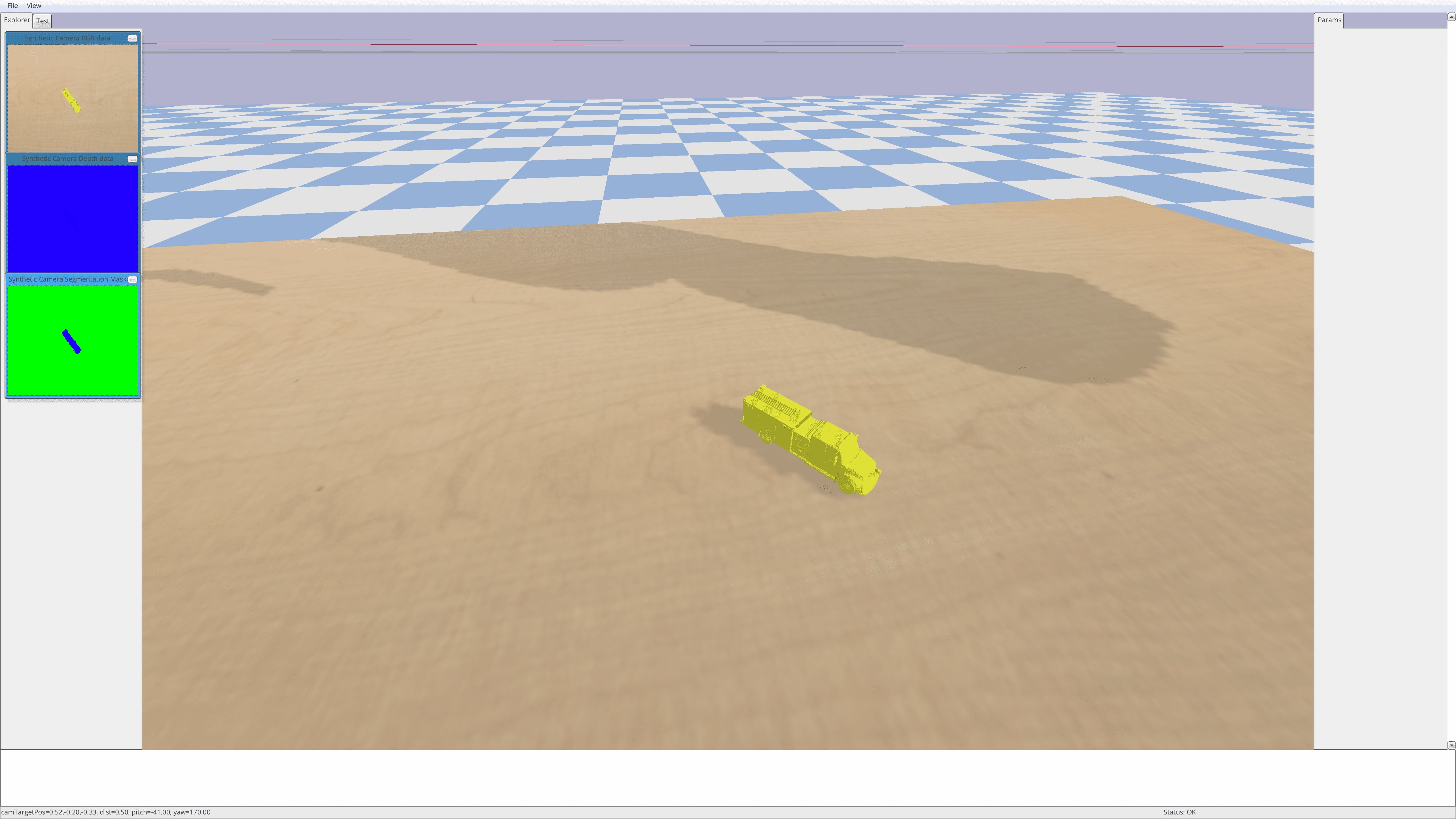}}
   \subfigure
   {\includegraphics[trim={85cm 42cm 70cm 38cm},clip,width=0.18\textwidth]{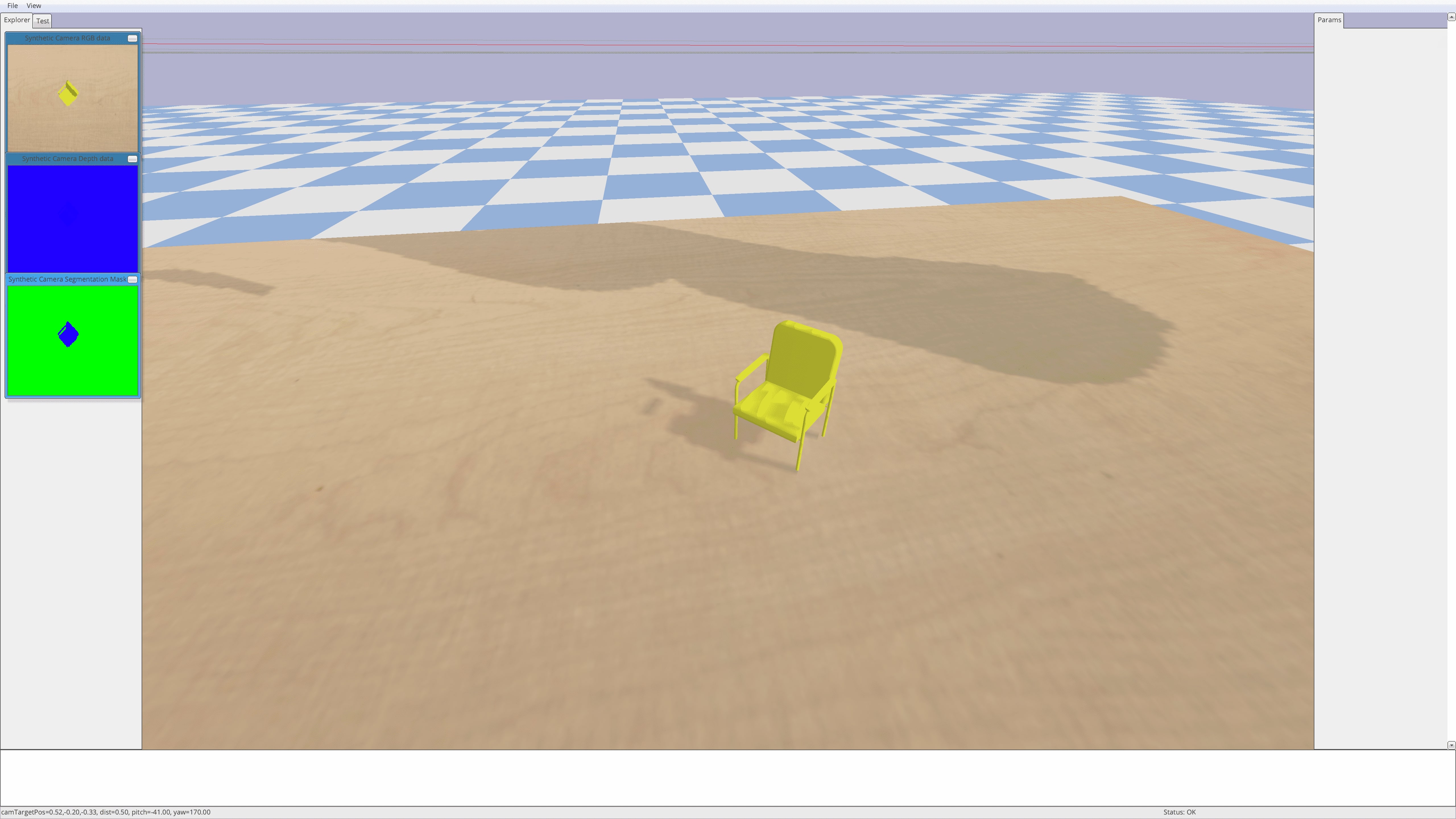}}
   \subfigure
   {\includegraphics[trim={90cm 40cm 65cm 40cm},clip,width=0.18\textwidth]{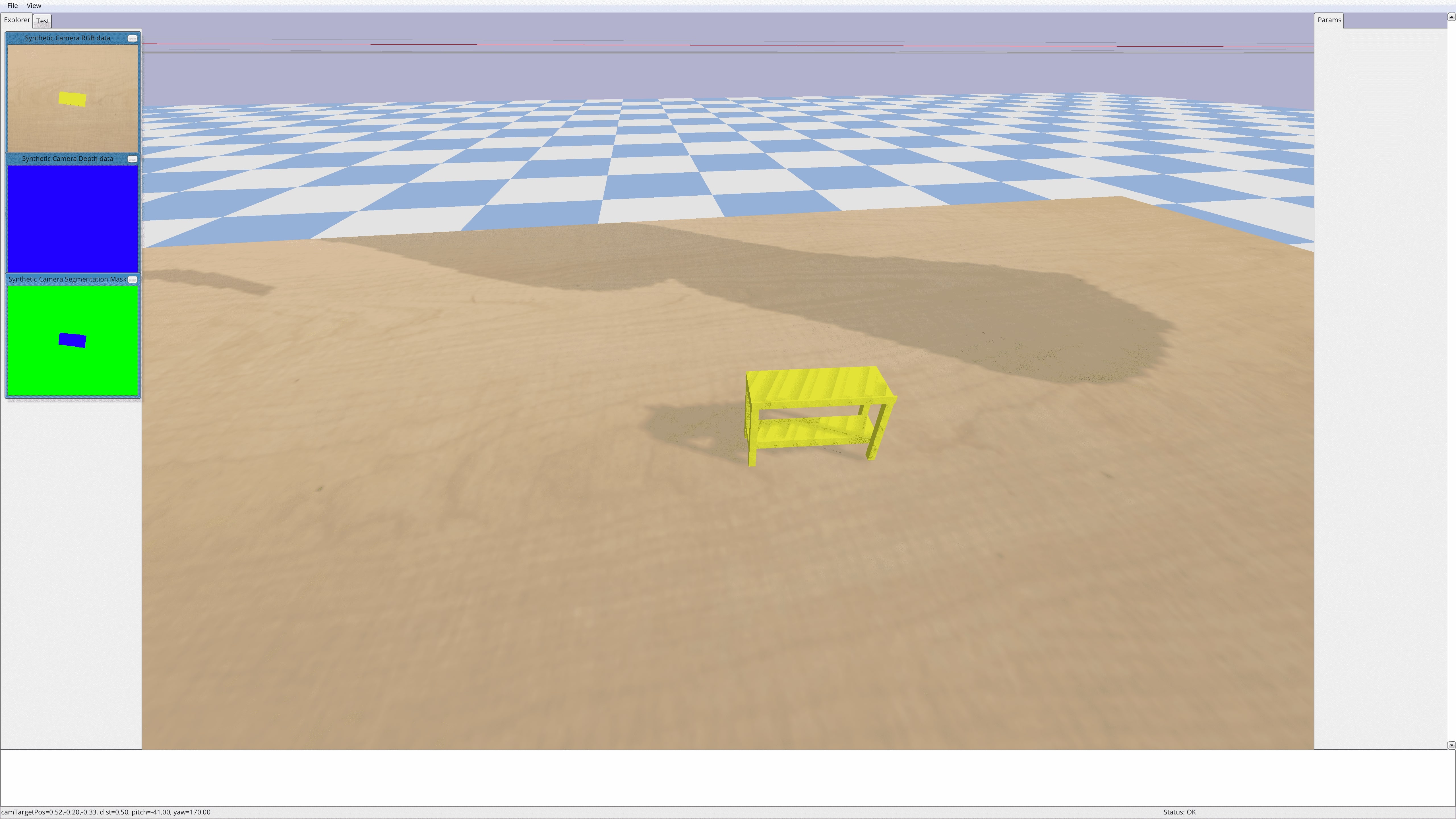}}
   \subfigure
   {\includegraphics[trim={84cm 40cm 71cm 40cm},clip,width=0.18\textwidth]{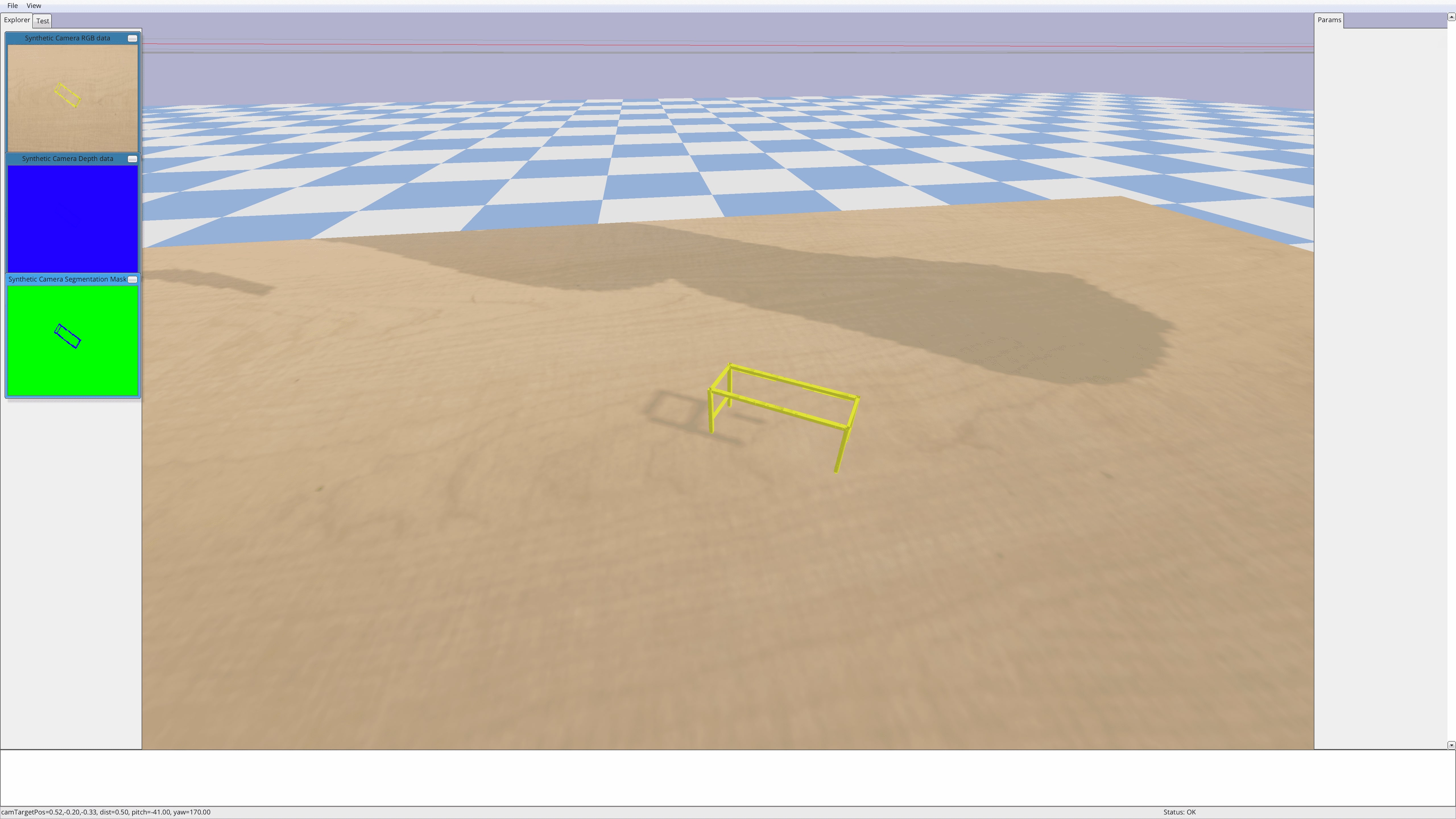}}
      \subfigure
   {\includegraphics[trim={90cm 41cm 65cm 39cm},clip,width=0.18\textwidth]{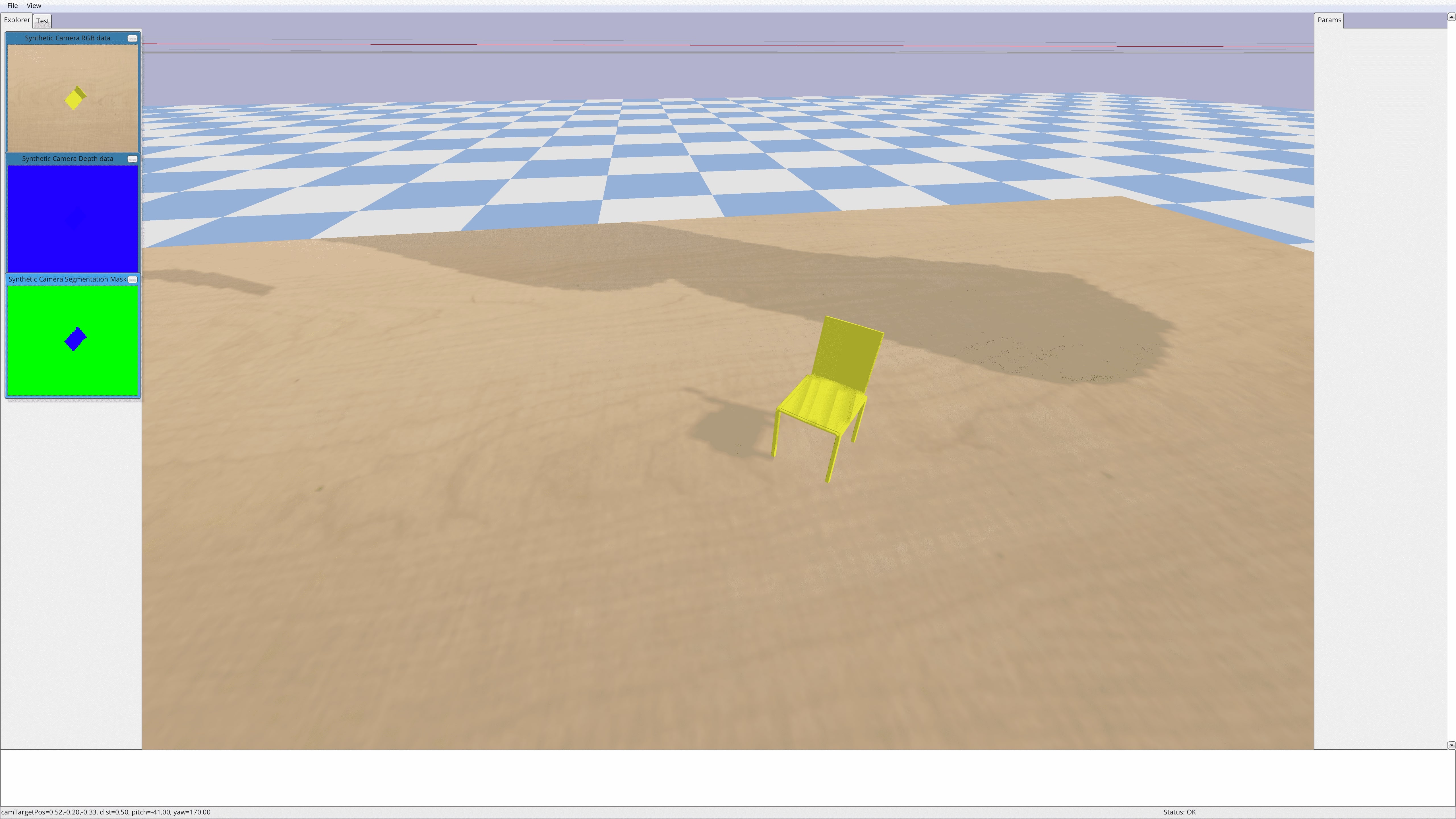}}
	\vspace{-5pt}
    \caption{\footnotesize{Representative examples of objects from the ShapeNet database. }\label{fig:shapenet}}
\end{figure}

{\bf Cost function.} We choose a cost function that assigns a cost of 0 if the robot successfully grasps the object and a cost of 1 otherwise. In particular, a ``successful" grasp is one that lifts the object to a certain height  ($>2$cm) above the table. 

\begin{figure}[t]
\begin{center}
\includegraphics[width=0.99\columnwidth]{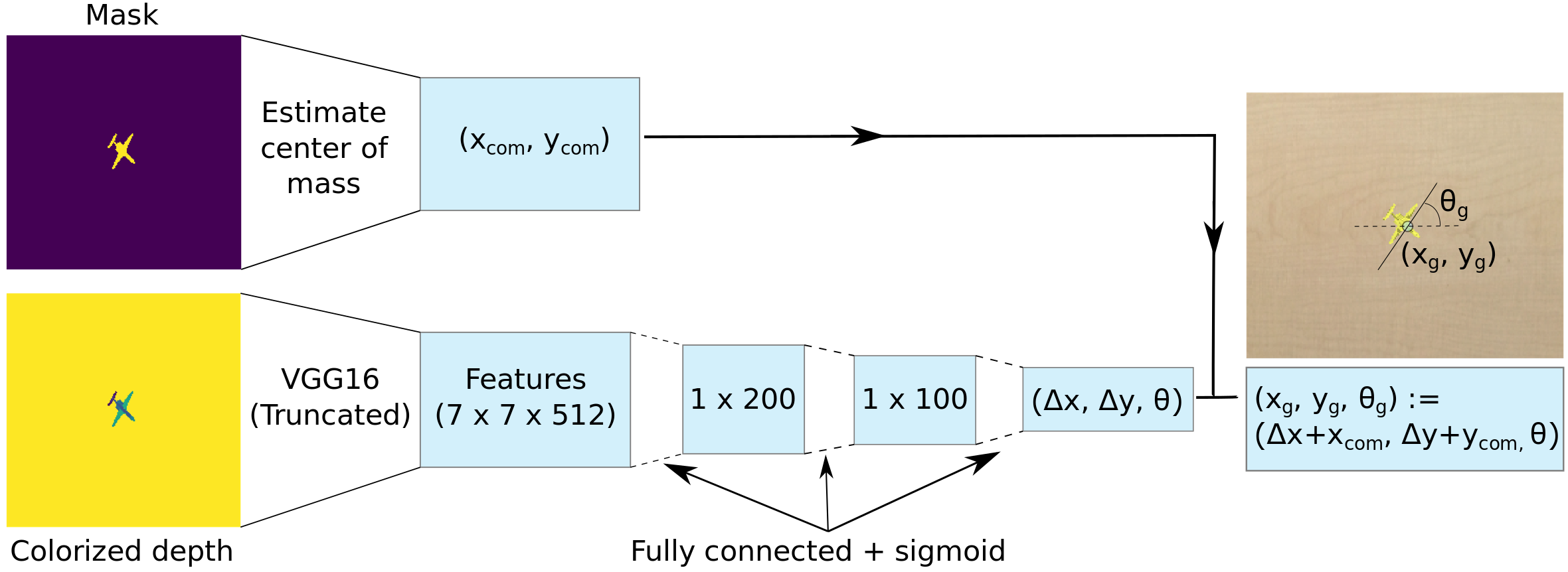}
\end{center} 
\caption{\footnotesize{The neural network-based architecture for our grasping policies.}  \label{fig:pipeline}}
\end{figure}

{\bf Neural network policy.} Our control policy maps a depth image of an object (and a corresponding mask image) to a grasp location $(x_g, y_g)$ and wrist angle $\theta_g$. A grasp is executed by servoing the robot's gripper to the grasp location $(x_g, y_g)$, setting the wrist angle to the desired angle $\theta_g$, and then executing an open-loop grasping maneuver that closes the grippers and lifts the robot arm up. 

The architecture for the pipeline that maps depth images (and corresponding masks) to grasps is illustrated in Figure \ref{fig:pipeline}. The mask image is used to estimate the center of mass (COM) $(x_\text{com}, y_\text{com}$) of the object (by simply computing the centroid of the object). Following \citep{Eitel15}, the raw depth image is \emph{colorized} via a jet colormap. This transforms the depth image from a single-channel image into a three-channel (RGB) image, thus allowing us to re-use neural network architectures pre-trained on the ImageNet dataset \citep{Deng09}. In particular, we pass the colorized depth image through a VGG16 network \citep{Simonyan14} pretrained on ImageNet and truncated to output a feature representation of size $7 \times 7 \times 512$. This feature vector is passed through three fully-connected layers with sigmoid activation. The (distributions over) weights of these fully-connected layers (represented in Figure \ref{fig:pipeline} using dashed lines) are learned using the training procedure described below. The output of this pipeline is a $1 \times 3$ vector $(\Delta x, \Delta y, \theta)$, which is combined with the estimated center of mass in order to obtain the final target grasp location and orientation $(x_g, y_g, \theta_g) := (\Delta x + x_\text{com}, \Delta y + y_\text{com}, \theta)$. 

{\bf Training.} We apply the procedure described in Section \ref{sec:continuous policy space} for training our stochastic control policy. In particular, we define Gaussian distributions $\N_{\mu, s}$ over the weights (and biases) of the fully-connected layers and choose a Gaussian distribution $\N_{\mu_0, s_0}$ as our prior distribution. Our particular choice of prior is motivated by the fact that the COM is a reasonable grasp position (i.e., $(\Delta x, \Delta y) = (0, 0)$ is a good guess for the grasp position in the absence of any further knowledge). In particular, we set the prior mean $\mu_0$ by randomly sampling from the distribution $\N_{0, 0.001^2}$. Note that while we could have chosen $\mu_0$ to be zero, a randomly chosen $\mu_0$ helps in breaking symmetries in the network (see \citep[Appendix B]{Dziugaite17} for a thorough discussion of this point). The prior variance $s_0$ is set to $0.01$.

For the purpose of optimization via stochastic gradient descent, we employ a differentiable surrogate cost function in place of the discontinuous 0-1 cost. In particular, for each object $\mathcal{O}_i$ in our training dataset, we exhaustively attempt $750$ grasps by discretizing the space $(\Delta x, \Delta y, \theta) \in [-0.05 \  \text{cm}, 0.05 \ \text{cm}] \times [-0.05 \ \text{cm}, 0.05 \ \text{cm}] \times [0 \ \text{rad}, \pi \ \text{rad}]$ into $5 \times 5 \times 30$ points. As before, $(\Delta x, \Delta y)$ denotes a perturbation from the estimated centroid of the object. For each of the $750$ grasps, we record whether the grasp succeeded or failed. We then choose the most ``robust" grasp for the given training object by selecting the grasp $(\Delta x^\star, \Delta y^\star, \theta^\star)$ that is most tolerant to errors in $\theta$ (i.e., the grasp for which one can perturb $\theta$ by the largest magnitude and still successfully grasp the object). This heuristic measure of robustness is motivated by our empirical observation that, in our setting, changes in grasp orientation have a very large impact on whether a grasp succeeds or not, while success is less sensitive to changes in the grasp position. 
 Our surrogate cost function is then computed as the magnitude of the difference between $(\Delta x^\star, \Delta y^\star, \theta^\star)$ and the grasp $(\Delta x, \Delta y, \theta)$ predicted by our neural network policy (note that this difference must take into account the fact that $\theta$ lies on a circle and must also be scaled to lie between $[0, 1]$ since costs in our framework are assumed to take values in this range). Importantly, we employ this surrogate cost \emph{only} for optimization; all bounds and results are presented for the 0-1 cost. 

{\bf Results.} We use $N = 2000$ objects randomly selected from the ShapeNet database as our training objects. We choose confidence parameters $\delta = 0.009, \delta' = 0.001$, and use $L = 1000$ samples to evaluate the sample convergence bound in equation \eqref{eq:sample bound}. The resulting PAC-Bayes bound $C^\star_\text{bound}$ is $0.294$. Thus, with probability $0.99$ over sampled training data, the optimized PAC-Bayes control policy is guaranteed to have an expected success rate of $70.6 \%$ on novel objects (assuming that they are drawn from the same underlying distribution as the training examples). We hypothesize that this bound could be further improved by using a larger number of samples $L$ in order to evaluate the sample convergence bound in equation \eqref{eq:sample bound} (this would come at an increased computational cost). 

We evaluated our PAC-Bayes policy on $1000$ test objects (unseen in the training phase). The policy was successful on $82.0\%$ of these objects. Videos from representative trials on test objects can be found at \href{https://youtu.be/NGI0_oXBdqw}{https://youtu.be/NGI0\_oXBdqw}.

We also compared our learned policy with a (deterministic) neural network policy trained by minimizing the training cost (i.e., without the regularization that comes from PAC-Bayes). We used an architecture that is identical to the PAC-Bayes policy and initialized weights for the network in the same manner as well (by using the means of the distribution used to define the initialization of the stochastic PAC-Bayes policy). The success rate for the resulting policy on test objects is approximately $78.0\%$ (as compared to $82.0\%$ for the PAC-Bayes policy). We thus see that without the regularization term from PAC-Bayes, the learned policy overfits to a larger degree. We also note that in addition to a loss in the empirical performance, simply minimizing the training cost does not allow us to obtain guarantees on generalization performance.

\section{\revision{Hardware Implementation}}
\label{sec:hardware}
\revision{In this section, we present results from hardware experiments aimed at validating our approach. The hardware platform we use is the Parrot Swing drone (Figure \ref{fig:lab_swing}). This lightweight (75g) quadrotor/fixed-wing hybrid vehicle is an appealing platform since it combines vertical take-off and landing with horizontal flight (thus making it more efficient than a traditional quadrotor configuration). We implement our approach from Section \ref{sec:finite policy space} (finite policy spaces) to achieve obstacle avoidance on different environments.}



\revision{{\bf Experimental setup.} The Swing takes off from one end of a netted area (see Figure \ref{fig:swingb}) and travels at a fixed speed of $2.0$ m/s. To achieve a cost of 0, the Swing must avoid large, cylindrical obstacles over a time horizon of $5$ seconds and land safely; otherwise, the Swing will incur a cost of 1. Additionally, we consider the net encompassing the area ($7\text{m} \times 18\text{m}$) as ``wall" obstacles. We use a Vicon motion tracking system to track the obstacle and Swing's locations.  
Since the Swing does not possess any sensors for detecting obstacles, we \emph{simulate} a 40-ray depth sensor as if it were mounted on the Swing. This is done using the locations of obstacles, walls, and Swing reported by the motion capture system. Thus, the Swing \emph{only} uses real-time information from this simulated depth sensor; we do not provide any additional information (e.g., the Swing or obstacle locations, etc.).  
These sensor measurements are provided to a ``ground" computer that calculates the control input given a policy. The control inputs to the Swing correspond to percentages of maximum roll, pitch, and yaw angles, as well as the vertical position of the Swing. Commands are sent to the Swing at 10 Hz via bluetooth using the PyParrot python library \citep{pyparrot}. We implement a reactive obstacle avoidance policy (identical to the one described in Section \ref{sec:reactive obstacle avoidance control}) on the Swing for the discrete policy space setting. }

\revision{{\bf Dynamics model.} We train our policies by minimizing the PAC-Bayes bound in simulation; the learned distribution over policies is then implemented on the Swing hardware for validation (on environments not seen during training). We first performed system identification on the Swing in order to obtain an accurate dynamics model. 
If we keep the Swing's speed and vertical position constant, we can use a simple model for its dynamics similar to the one for the ground vehicle with states $[x,y,\psi]$ (Section \ref{sec:reactive obstacle avoidance control}). We can keep the Swing's speed constant by fixing its pitch angle $\theta$. Thus we fix the vertical position to $1$ m, and $\theta = 27^\circ$; the Swing will then travel at about $u_0 = 2.0$ m/s. Consider the following dynamics: }

\begin{equation}
\revision{
\left[ \begin{array}{c}
\dot{x} \\
\dot{y} \\
\dot{\psi} \end{array} \right] = 
 \left[ \begin{array}{c} 
-u_0\sin(\psi) \\
u_0\cos(\psi) \\
k_p  (k_u u_\psi - \psi) \end{array} \right],}
\label{eq:husky dynamics}
\end{equation}
\revision{where the only control input $u_\psi$ is a percentage of the Swing's maximum yaw angle $\psi$, and $k_p$ and $k_u$ are gains. Both $k_p$ and $k_u$ are fit with empirical data to create a realistic simulation. The gain $k_u$ is needed to scale $u_\psi$ such that $k_u u_\psi = \psi$ if $u_\psi$ remains constant and a steady state is reached. We limit $k_u u_\psi \in [-\frac{\pi}{4}, \frac{\pi}{4}]$ to restrict the Swing to maneuvers that do not significantly change the Swing's forward velocity or vertical position. We first determine $k_u$ by measuring (with the Vicon motion tracking system) the steady state yaw angle given a fixed control input. We then model the proportional gain $k_p$ with varied input signals such as sinusoidal and chirp functions of varying amplitude. The resulting dynamics, given by $k_u = 3.0$ and $k_p = 0.4$, are implemented in a simulated PyBullet environment analogous to the one described for the ground vehicle.}

\revision{{\bf Results.} We choose a time horizon $T = 50$; the Swing then flies for $5s$. We then choose $L = 100$ different $K$'s in the form of \eqref{eq:K} with $(x_0, y_0)$ chosen by discretizing the space $[0.1, 5.0] \times [0, 60.0]$ into $5$ and $20$ values for $x_0$ and $y_0$ respectively. As with the method for the ground vehicle in Section \ref{sec:reactive obstacle avoidance control}, we find an upper bound $C_{\text{bound}}^*$ on the true expected cost of the PAC-Bayes control policy $P_{\text{PAC}}^\star$ using Algorithm \ref{a:pac bayes control}. With 1000 training environments, $P_{\text{PAC}}^\star$ is guaranteed (with 99$\%$ probability) to succeed on new environments 88.6$\%$ of the time. The empirical success rate, tested on Swing hardware in unseen real-world environments, is approximately 90$\%$ (18/20 trials). Videos of representative trials can be found at https://youtu.be/p5CjcSsojg8.}

\section{Discussion and Conclusions}
\label{sec:conclusions}

We have presented an approach for learning control policies that provably generalize well to novel environments given a dataset of example environments. Our approach leverages PAC-Bayes theory to obtain upper bounds on the expected cost of (stochastic) policies on novel environments and can be applied to robotic systems with continuous state and action spaces, complicated dynamics, rich sensory inputs, and neural network-based policies. We synthesize policies by explicitly minimizing this upper bound using convex optimization in the case of a finite policy space, and using stochastic gradient descent in the more general case of continuously parameterized policies. We also present an extension of our approach for learning distributionally-robust policies, i.e., settings where test environments are drawn from a different distribution than training environments. We demonstrated our framework by learning (i) depth sensor-based obstacle avoidance policies with guarantees on collision-free navigation in novel environments, and (ii) neural network-based grasping policies with guarantees on generalization to new objects. Our simulation results compared the generalization guarantees provided by our technique with exhaustive numerical evaluations in order to demonstrate that our approach is able to provide strong bounds even with relatively few training environments. \revision{Our hardware experiments -- which tested policies learned using our framework on a real-world obstacle avoidance example -- suggest that our technique is effective for developing policies that generalize well to (unseen) real-world environments.} \revisionn{We believe that taken together, the simulation and hardware results provide significant evidence for the ability of our approach to provide strong generalization guarantees in realistic robot control settings.}

\subsection{Challenges and Future Work}
\label{sec:future work}

There are a number of challenges and exciting opportunities for future work on both the theoretical and practical fronts. We highlight a few such directions here.

{\bf Deterministic policies.} It may be desirable in many cases (e.g., safety-critical settings) to learn deterministic policies instead of stochastic ones. Techniques for converting stochastic hypotheses into deterministic hypotheses have been developed within the PAC-Bayes framework (e.g., using majority voting in the classification setting \citep{Langford03, Lacasse07}); an interesting avenue for future work is to extend such techniques to the policy learning setting we consider here. Another possibility is to use different frameworks for obtaining generalization bounds that are better suited to deterministic policies (e.g., bounds based on algorithmic stability \citep{Bousquet02, Kearns99, Hardt15} and sample compression \citep{Floyd95, Langford05}). An important feature of the reduction-based perspective we presented in Section \ref{sec:pac bayes control} is that it immediately allows us to port over such bounds from the supervised learning setting to our setting.

{\bf Choosing the prior.} While we have demonstrated that our framework allows us to obtain strong bounds on generalization performance, an important direction for future work is to find ways to further improve these bounds. We believe that a particularly promising approach for doing this is to systematically choose the prior $P_0$ over the control policy space. \revision{The ability to specify a strong prior is an important distinction between the robot control settings considered in this paper and standard supervised learning problems (e.g., image recognition). For standard supervised learning problems, it is often challenging to specify a prior over the space of hypotheses. While the priors in the examples considered in this paper were chosen in a fairly simplistic manner (e.g., a uniform prior over the finite policy space for the obstacle avoidance example in Section \ref{sec:reactive obstacle avoidance control}, or a prior that attempts to keep the grasp position close to the center of mass in the grasping example in Section \ref{sec:grasping}), we believe that choosing priors in a more systematic manner could significantly improve the generalization bounds.} One possibility for choosing a prior more carefully is to embed domain knowledge into the prior; for example, one could choose a prior that incorporates a physics model of the system, or one that is derived from an existing state-of-the-art approach for the problem under consideration (e.g., choosing a prior that encourages force-closure grasps). Another promising possibility is to learn the prior from a human expert using imitation learning. \revision{By incorporating such priors for robot control problems, we may need significantly smaller datasets than the ones currently used to train state-of-the-art supervised learning models while still obtaining strong generalization guarantees.}

\revision{{\bf Incorporating different regularizers.} The algorithmic approach we employ in this work (Section \ref{sec:computing pac-bayes controllers}) involves minimizing a combination of the training cost and a regularizer specified by PAC-Bayes theory. This is motivated by the desire to optimize the PAC-Bayes upper bound on the expected cost on novel environments. However, we note that there are a variety of regularization techniques that have been empirically demonstrated to promote generalization (e.g., dropout \cite{Srivastava14} and overparameterization \cite{Neyshabur14, Zhang16, Arora18}), in addition to other techniques such as domain randomization \cite{Tobin17} and batch normalization \cite{Ioffe15}. While these techniques do not yet have strong generalization bounds associated with them, there is a growing literature on this topic \cite{McAllester13, Arora18, Li18, Bjorck18}. Incorporating different regularization schemes into our framework while maintaining strong generalization guarantees is a promising direction for future work.}


{\bf Extensions to meta-learning.} Another exciting future direction is to combine the techniques presented here with \emph{meta-learning} techniques in order to achieve provably data-efficient control on novel tasks. Specifically, we are currently investigating using a PAC-Bayes bound as part of the objective of a meta-learning algorithm such as MAML \citep{Finn17a} to achieve improved generalization performance and few-shot learning. 

\vspace{7pt}

We believe that the approach presented here along with the indicated future directions represent an important step towards learning control policies with provable guarantees for challenging robotic platforms with rich sensory inputs operating in novel environments.

\vspace{-7pt}
\section*{Acknowledgements}

The authors are grateful to Max Goldstein for initiating the grasping example in Section \ref{sec:grasping} and contributions to the conference version of this paper presented at CoRL 2018. We also gratefully acknowledge the support of NVIDIA Corporation with the donation of the Titan Xp GPU used for this research. 

\vspace{-5pt}
\section*{Funding}

The authors were partially supported by the Office of Naval Research [Award Number: N00014-18-1-2873], the National Science Foundation [IIS-1755038], the Google Faculty Research Award, and the Amazon Research Award.

\vspace{-5pt}

\bibliographystyle{abbrvnat} 
\bibliography{irom.bib}

\end{document}